\documentclass{article}

\usepackage[preprint]{neurips_2019}

\usepackage{amsmath, mathtools}
\usepackage{amsthm}
\usepackage{newtxmath}
\usepackage[T1]{fontenc}
\usepackage[utf8]{inputenc}
\usepackage{bm}
\usepackage{graphicx}
\usepackage{microtype}

\makeatletter
\theoremstyle{plain}
\newtheorem{thm}{\protect\theoremname}
\theoremstyle{plain}
\newtheorem{cor}[thm]{\protect\corollaryname}
\theoremstyle{plain}
\newtheorem{lem}[thm]{\protect\lemmaname}
\theoremstyle{remark}
\newtheorem{rem}[thm]{\protect\remarkname}

\newcommand{\D}{\mathrm{d}}
\newcommand{\E}{\mathrm{e}}
\newcommand{\I}{\mathrm{i}}
\newcommand{\e}{\varepsilon}

\newcommand{\R}{\mathbb{R}}
\newcommand{\Z}{\mathbb{Z}}
\newcommand{\F}{\mathcal{F}}

\newcommand{\uhat}{\hat{u}}
\newcommand*\diff{\mathop{}\!\D}

\newcommand{\abs}[1]{\left\lvert#1\right\rvert} 
\newcommand{\norm}[1]{\left\lVert#1\right\rVert}

\usepackage{changepage}

\usepackage{url}% simple URL typesetting
\usepackage{booktabs}% professional-quality tables
\usepackage{amsfonts}% blackboard math symbols
\usepackage{nicefrac}% compact symbols for 1/2, etc.
\usepackage{caption}
\usepackage{xcolor}

\author{
Yaoyu Zhang \\
New York University Abu Dhabi  and Courant Institute of Mathematical Sciences\\
\And 
Zhi-Qin John Xu\thanks{Corresponding author: zhiqinxu@nyu.edu}\\
New York University Abu Dhabi and Courant Institute of Mathematical Sciences \\
\And
  Tao Luo\\
  Department of Mathematics, Purdue University\\
\And
  Zheng Ma\\
  Department of Mathematics, Purdue University\\
}

\providecommand{\theoremname}{Theorem}

\providecommand{\corollaryname}{Corollary}
\providecommand{\lemmaname}{Lemma}
\providecommand{\remarkname}{Remark}
\providecommand{\theoremname}{Theorem}

\@ifundefined{showcaptionsetup}{}{%
 \PassOptionsToPackage{caption=false}{subfig}}
\usepackage{subfig}
\makeatother

\providecommand{\corollaryname}{Corollary}
\providecommand{\lemmaname}{Lemma}
\providecommand{\remarkname}{Remark}
\providecommand{\theoremname}{Theorem}

\begin{document}

\title{Explicitizing an Implicit Bias of the Frequency Principle in Two-layer
Neural Networks}

\maketitle
\begin{abstract}
It remains a puzzle that why deep neural networks (DNNs), with more
parameters than samples, often generalize well. An attempt of understanding
this puzzle is to discover implicit biases underlying the training
process of DNNs, such as the Frequency Principle (F-Principle), i.e.,
DNNs often fit target functions from low to high frequencies. Inspired by the F-Principle, we propose an effective model of linear F-Principle (LFP) dynamics which accurately predicts the learning results of two-layer ReLU neural networks (NNs) of large widths. This LFP dynamics is rationalized by a linearized mean field residual dynamics of NNs. Importantly, the long-time limit solution of this LFP dynamics is equivalent to the
solution of a constrained optimization problem explicitly minimizing an FP-norm, in which higher frequencies of feasible solutions
are more heavily penalized. Using this optimization formulation, an \emph{a priori} estimate of the generalization error bound is provided, revealing that a higher FP-norm of the target function increases the generalization error. Overall, by explicitizing the implicit
bias of the F-Principle  as an explicit penalty for two-layer NNs, our work makes a step towards a quantitative understanding of the learning and generalization of general DNNs. 
\end{abstract}

\section{Introduction\label{sec:Introduction}}

The wide success of deep learning in many fields \citep{lecun2015deep}
remains a mystery. For example, a puzzle recently attracts a lot of
attention, that is, why Deep Neural Networks (DNNs), with more parameters
than samples, often generalize well \citep{zhang2016understanding}.
A major difficulty of resolving this puzzle may be attributed to the
lack of an effective model which can accurately predict the final
output function of DNNs and yet is simple enough for analysis. Devising
such an effective model could propel deep learning into a new era
in which quantitative understandings of deep learning replace the
qualitative or empirical ones.

Towards this end, we begin with a widely observed phenomenon of DNNs,
that is, Frequency Principle (F-Principle) \citep{xu_training_2018,rahaman2018spectral,xu2019frequency}:

\begin{changemargin}{0.5cm}{0.5cm}\emph{DNNs initialized with
small parameters often fit target functions from low to high frequencies
during the training.} \end{changemargin}

Without an explicit mathematical description, it is unclear how this
implicit bias of the F-Principle functions quantitatively during the
training. Inspired by the F-Principle, we construct a model of linear F-Principle
(LFP) dynamics, which explicitly imposes different priorities on different
frequencies in the gradient flow dynamics. Experimentally, we show
that the LFP model can accurately predict the output of two-layer
ReLU neural networks (NNs) of large widths. We then rationalize
the LFP model using a linearized mean field residual dynamics of DNNs, which is widely considered in recent theoretical studies of
DNNs \citep{mei2018mean,rotskoff2018parameters,mei2019mean}. We prove that the long-time limit solution
of this LFP dynamics is equivalent to the solution of a constrained
optimization problem minimizing an F-Principle norm (FP-norm), in
which higher frequencies of feasible solutions are more heavily penalized.
Therefore, by analyzing the explicit regularity
underlying the FP-norm, we can obtain a quantitative understanding
of the behavior of two-layer NNs. 

With a reasonable construction process and an ability of making accurate
predictions for two-layer ReLU NNs of large widths, the
LFP model qualifies as a primitive candidate of an effective model
of DNNs, which is capable of providing quantitative understandings
of deep learning. To analyze the generalization error of the LFP model, we first use the FP-norm, the explicit penalty, to induce an FP function space, and estimate its Rademacher complexity. We then provide an \emph{a priori} estimate, i.e., an estimate
without the knowledge of the model solution, of the generalization error
for the LFP model, which is bounded by the FP-norm of the target function,
scales as $1/\sqrt{M}$ as the number of training samples $M$ increases,
and is independent of the number of parameters in NNs. 

\section{Related works}

Various approaches have been proposed in an attempt to resolve the
generalization puzzle of DNNs. For example, the generalization error
has been related to various complexity measures (\citealp{bartlett1999almost,bartlett2002rademacher,bartlett2017spectrally,bartlett2017nearly,neyshabur2017exploring,golowich2017size,dziugaite2017computing,neyshabur2018towards,ma2018priori}),
local properties (sharpness/flatness) of loss functions at minima
(\citealp{hochreiter1995simplifying,keskar2016large,dinh2017sharp,wu2017towards}),
stability of optimization algorithms (\citealp{bousquet2002stability,xu2012robustness,hardt2015train}),
and implicit bias of the training process (\citet{arpit2017closer,rahaman2018spectral,xu2018understanding,xu2018frequency,perez2018deep,xu2019frequency,xu_training_2018,neyshabur2014search,poggio2018theory,soudry2018implicit}).
Recently, analyzing DNNs in an extremely over-parameterized regime renders a promising approach.
For example, the training process of two-layer neural networks at
the mean-field limit can be described by a partial differential equation
\citep{rotskoff2018parameters,mei2018mean,sirignano2018mean}. In
addition, the training dynamics of a DNN in an extremely over-parameterized
regime is found to be well approximated by the gradient flow of a
linearized model of the DNN resembling kernel methods \citep{jacot2018neural,lee2019wide}. This result initiates a series of works. For example, \cite{arora2019fine,cao2019generalization,E2019analysis,E2019comparative}
utilize the linearized model to study the generalization error bounds
of DNN. Note that an \emph{a priori} generalization error bound for
two-layer NNs is provided in \citet{ma2018priori}, in which an explicit
penalty is imposed to the loss function of NNs. In contrast, our \emph{a
prior} generalization bound works for NNs without any extra penalty. 

\section{Notation and experimental setup}

\subsection{Notation}

For a two-layer neural network, its output (also known as the hypothesis function) reads as
\begin{equation}
    h(x, \theta) = \sum_{i=1}^N w_i \sigma\left(r_i\cdot x - \abs{r_i}l_i\right), \label{eq:relunn2}
\end{equation}
where $r_i, x \in \R^d$, $\theta = (w,R,l)$, $w,l\in\R^{N}$
and $R\in\R^{N\times d}$, and by default
$\sigma(x)=\max(x,0)$ is the activation function of ReLU. The target
function is denoted by $f(x)$. 
In this work, we consider the mean-squared error (MSE) loss function 
\begin{equation}
  L=\int_{\R^d}\frac{1}{2}\abs{h(x,\theta)-f(x)}^{2}p(x) \diff{x}, \label{eq:mseloss}
\end{equation}
where $p(x)$ is the population probability density. 
The following notation will be used in studying the training dynamics:
$u(x,t)= h(x,\theta(t))-f(x)$, $u_{p}(x,t)= h_{p}(x,\theta(t))-f_{p}(x)$,
where $h_{p}(x,\theta(t))=h(x,\theta(t))p(x)$, $f_{p}(x)=f(x)p(x)$. In this work, we focus on $p(x)=\frac{1}{M}\sum_{i=1}^M\delta(x-x_i)$, which accounts for the real case of a finite training dataset $\{x_i;y_i\}_{i=1}^M$ with each input $x_i\in\R^d$ and output $y_i\in\R$. Because the target function $f(x)$ is not available, without loss of generality, we fixed it to any continuous function satisfying $f(x_i)=y_i$ for $i=1,\cdots,M$. Then $f_{p}(x)=\frac{1}{M}\sum_{i=1}^My_i\delta(x-x_i)$. Because $\partial_{t}f(x)=0$ for any fixed $f(x)$, its choice does not affect the training dynamics of $h(x,\theta(t))$.

For any function $g$ defined on $\R^n$, $n\in \Z^+$, we use the following convention of the Fourier transform and its inverse:
\begin{equation*}
\textstyle{  
\mathcal{F}[g](\xi)=\hat{g}(\xi)=\int_{\R^n}g(x)\E^{-2\pi\I \xi\cdot x}\diff{x},\quad
  g(x)=\int_{\R^n}\hat{g}(\xi)\E^{2\pi\I x\cdot \xi}\diff{\xi},
}
\end{equation*}
where $\xi\in\R^{n}$ denotes the frequency. For any function $g$ defined on the torus $\Omega:=\mathbb{T}^n=[0,1]^n$, $n\in \Z^+$, we use the following convention of the Fourier series and its inverse:
\begin{equation*}
\textstyle{
\hat{g}(k)=\int_{\mathbb{T}^n}g(x)\E^{-2\pi\I k\cdot x}\diff{x},\quad
  g(x)=\sum_{k\in\Z^n}\hat{g}(\xi)\E^{2\pi\I x\cdot k},
}
\end{equation*}
where $k\in\Z^{n}$ denotes the frequency. For $u$ and $u_p$, their Fourier transforms are written as $\hat{u}(\xi,t)$
and $\widehat{u_{p}}(\xi,t)$, respectively. $\theta(t)=(w(t),R(t),l(t))$
are the parameters at training time $t$. 

\subsection{Experimental setup}

In our experiments, we use two-layer ReLU NNs of form $h(x,\theta)=\sum_{i=1}^{N}w_{i}\sigma(r_{i}\cdot x-|r_{i}|l_{i})$ for input dimension $d>1$ and $h(x,\theta)=\sum_{i=1}^{N}w_{i}\sigma(r_{i}(x-l_{i}))$ for $d=1$. The NNs are trained
with MSE loss and full batch size. The learning rate for Fig. \ref{fig:2relu}
and \ref{fig:2relu-1} is $3\times10^{-5}$, for Fig. \ref{fig:fpnorm}
is $10^{-4}$. The training algorithm for Fig. \ref{fig:2relu} and
\ref{fig:2relu-1} is gradient descent, for Fig. \ref{fig:fpnorm}
is Adam \citep{kingma2014adam}. $l_{i}$'s
are initialized by a uniform distribution on $[-1,1]$ for Fig. \ref{fig:2relu} and on $[-4,4]$ for Fig. \ref{fig:2relu-1}. For Fig. \ref{fig:2relu-1}, $w_{i}$'s and $r_{i}$'s are initialized
by ${\cal N}(0,1)$ and ${\cal N}(0,0.49)$. For Fig. \ref{fig:fpnorm}, We use an NN of $5000$ hidden neurons initialized
by Xavier normal initialization. 

A random non-zero initial output of DNN leads to a specific type of generalization error.
To eliminate this error, we use DNNs with an antisymmetrical initialization (ASI) trick\citep{zhang_type_2019}.

\section{An effective model of Linear F-Principle (LFP) dynamics}

It is difficult to analyze DNNs theoretically due to its huge number of parameters
and highly non-linear dynamics. In this section, inspired by the F-Principle,
we propose a Linear F-Principle (LFP) dynamics to effectively model
a two-layer ReLU NN of a large width. Specifically, with the loss
of MSE, up to a multiplicative constant in time scale, we model the gradient descent dynamics
of the two-layer NN of a sufficiently large width $N$ as
\begin{equation}
\partial_{t}\hat{u}(\xi,t)=-\left(\frac{\frac{1}{N}\sum_{i=1}^{N}\left(|r_{i}(0)|^{2}+w_{i}(0)^{2}\right)}{|\xi|^{d+3}}+\frac{4\pi^{2}\frac{1}{N}\sum_{i=1}^{N}\left(|r_{i}(0)|^{2}w_{i}(0)^{2}\right)}{|\xi|^{d+1}}\right)\widehat{u_{p}}(\xi),\label{eq:relunnFP}
\end{equation}
where, different from $\hat{u}(\xi,t)$ in the left hand side (LHS),
$\widehat{u_{p}}(\xi,t)=\mathcal{F}(u(\cdot,t)p(\cdot))=\mathcal{F}\left[\sum_{i=1}^{M}\left(h(x_{i},\theta(t))-y_{i}\right)\delta(\cdot-x_{i})\right]$
in the right hand side (RHS) incorportates the information from the
training dataset. 
In our numerical experiments, we only consider NNs with ASI trick\citep{zhang_type_2019}, which guarantees $h(\cdot,\theta_{0})=0$. Note that, for $d=1$, an NN of the form $h(x,\theta)=\sum_{i=1}^{N}w_{i}\sigma(r_{i}(x-l_{i}))$
is modeled by the same LFP dynamics \eqref{eq:relunnFP}. 
For convenience, we refer to the
long-time limit solution of the LFP dynamics as the solution of the
LFP model. Intuitively, the coefficient as a function of $\xi$ in
the RHS characterizes a decaying priority of convergence for $\hat{u}(\xi,t)$
from low to high frequencies, conforming with the phenomenon of the
F-Principle \citep{xu_training_2018,xu2019frequency}.  

Intuitively, a higher order of decay in the frequency domain, say
$1/|\xi|^{d+3}$ comparing to $1/|\xi|^{d+1}$, leads to a more ``smooth''
solution. Therefore, adjusting the relative importance of $1/|\xi|^{d+3}$
and $1/|\xi|^{d+1}$ through their coefficients $\frac{1}{N}\sum_{i=1}^{N}\left(|r_{i}(0)|^{2}+w_{i}(0)^{2}\right)$
and $4\pi^{2}\frac{1}{N}\sum_{i=1}^{N}\left(|r_{i}(0)|^{2}w_{i}(0)^{2}\right)$,
we can obtain solutions of different regularity/smoothness for a given
training dataset.

Before we rationalize this model, we first demonstrate the effectiveness
of this model through experiments on the synthetic training data of
1-d and 2-d input. Note that the long-time limit solution of dynamics (\ref{eq:relunnFP}) is obtained by solving an equivalent optimization
problem numerically (see Section \ref{sec:Explicitizing-the-implicit}
and Appendix \ref{sec:Numerical-solution-of} for details.).

For the case of 1-d input, i.e., $d=1$, we consider a training dataset
of $12$ samples as shown in Fig. \ref{fig:2relu}. We first initialize
$w_{i}$'s and $r_{i}$'s by uniform distributions on $[-0.1,0.1]$
and $[-0.25,0.25]$, respectively, such that $1/|\xi|^{4}$ dominates
in Eq. \eqref{eq:relunnFP}. As shown in Fig. \ref{fig:2relu}(a),
for the two-layer ReLU NN of $500$ hidden neurons, the solution
of the corresponding LFP model well approximates the final output
of the NN. As we increase the number of hidden neurons to $16000$,
as shown in Fig. \ref{fig:2relu}(b), the approximation becomes almost
perfect. We use $L^{p}$ norm to quantify the approximation error
at testing points $\{x_{i}\}_{i=1}^{M_{\mathrm{test}}}$, i.e., $L^{p}(h_{N},h_{\mathrm{LFP}})=\left(\sum_{i=1}^{M_{{\rm test}}}|h_{N}(x_{i})-h_{\mathrm{LFP}}(x_{i})|^{p}\right)^{1/p},$ where
$h_{N}$ is the final output of the two-layer ReLU NN of $N$ hidden
neurons, $h_{\mathrm{LFP}}$ is the solution of the corresponding
LFP model. As shown in Fig. \ref{fig:2relu}(c), $L^{1}(h_{N},h_{\infty})$
and $L^{2}(h_{N},h_{\infty})$ decrease as $N$ increases, indicating
that our LFP model accurately models the two-layer NNs of sufficiently
large widths. To obtain a less smooth interpolation of the training
data, we initialize $w_{i}$'s and $r_{i}$'s with uniform distributions
on $[-2,2]$, such that $1/|\xi|^{2}$ dominates in Eq. \eqref{eq:relunnFP}.
As shown in Fig. \ref{fig:2relu}(d), the solution of the LFP model
is less smooth compared to that in Fig. \ref{fig:2relu}(a) and (b).
In fact, it is close to a piecewise linear function. We note that
the LFP model with only the decaying term of $1/|\xi|^{2}$ in the
RHS of Eq. \eqref{eq:relunnFP} performs the piecewise linear interpolation
for $d=1$. We will elaborate this result in our future works. In
this example, the solution of our LFP model almost perfectly
overlaps with the final output of the NN.

For the case of 2-d input, i.e., $d=2$, we consider the training
dataset of the famous XOR problem, which cannot be solved by one-layer
neural networks. This training dataset consists of four points represented
by white stars in Fig. \ref{fig:2relu-1}(a). As shown in Fig. \ref{fig:2relu-1}(a-c),
our LFP model predicts the final output of NN very accurately over
the domain $[-1,1]^{2}$. Similar to the 1-d case, as the number
of hidden neurons increases, the prediction by the LFP model
becomes more accurate. We also present similar experimental results
for an asymmetrical training dataset in Fig. \ref{fig:2relu-1-1}
in Appendix. 

\begin{center}
\begin{figure}
\begin{centering}
\subfloat[]{\begin{centering}
\includegraphics[scale=0.2]{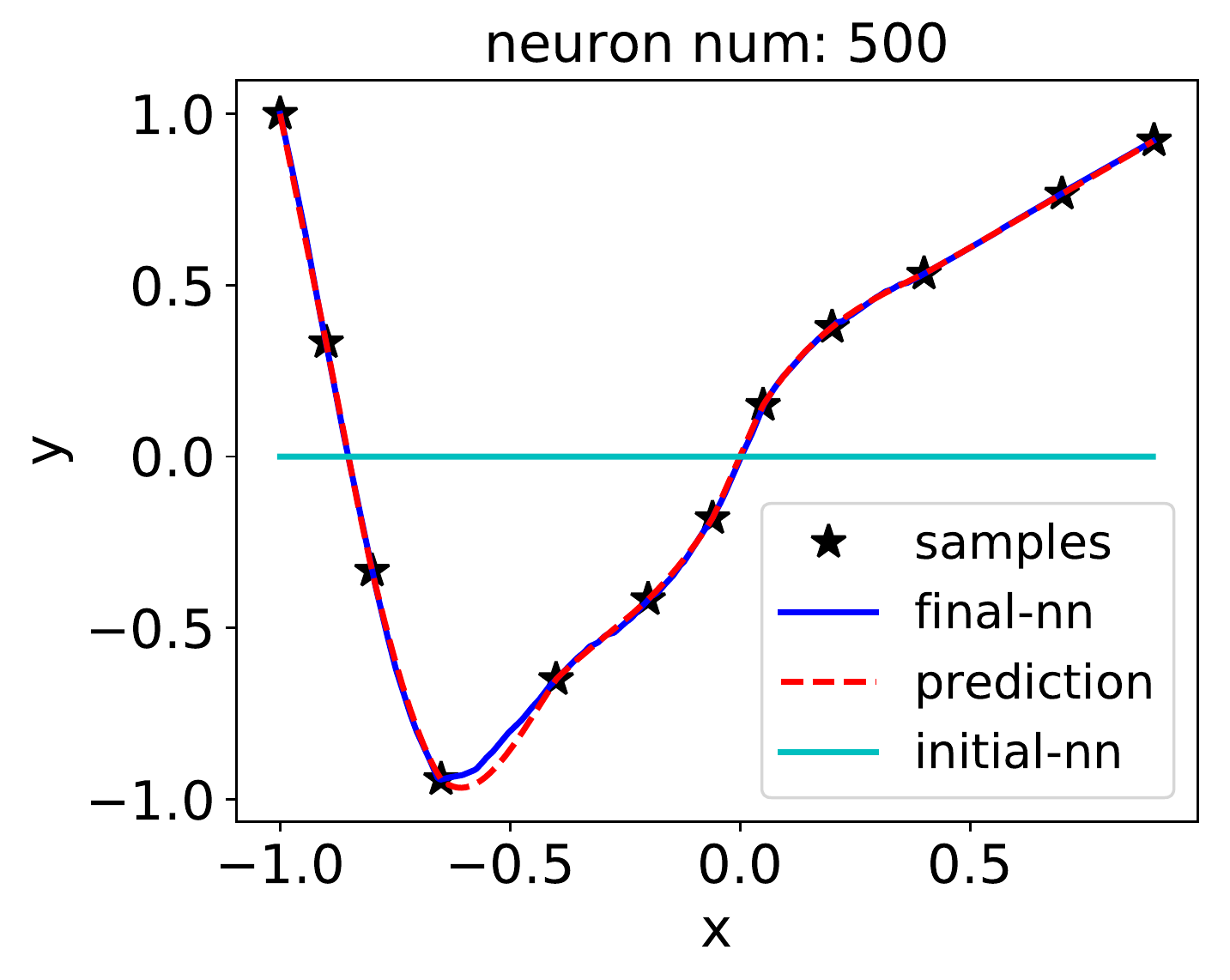} 
\par\end{centering}
}\subfloat[]{\begin{centering}
\includegraphics[scale=0.2]{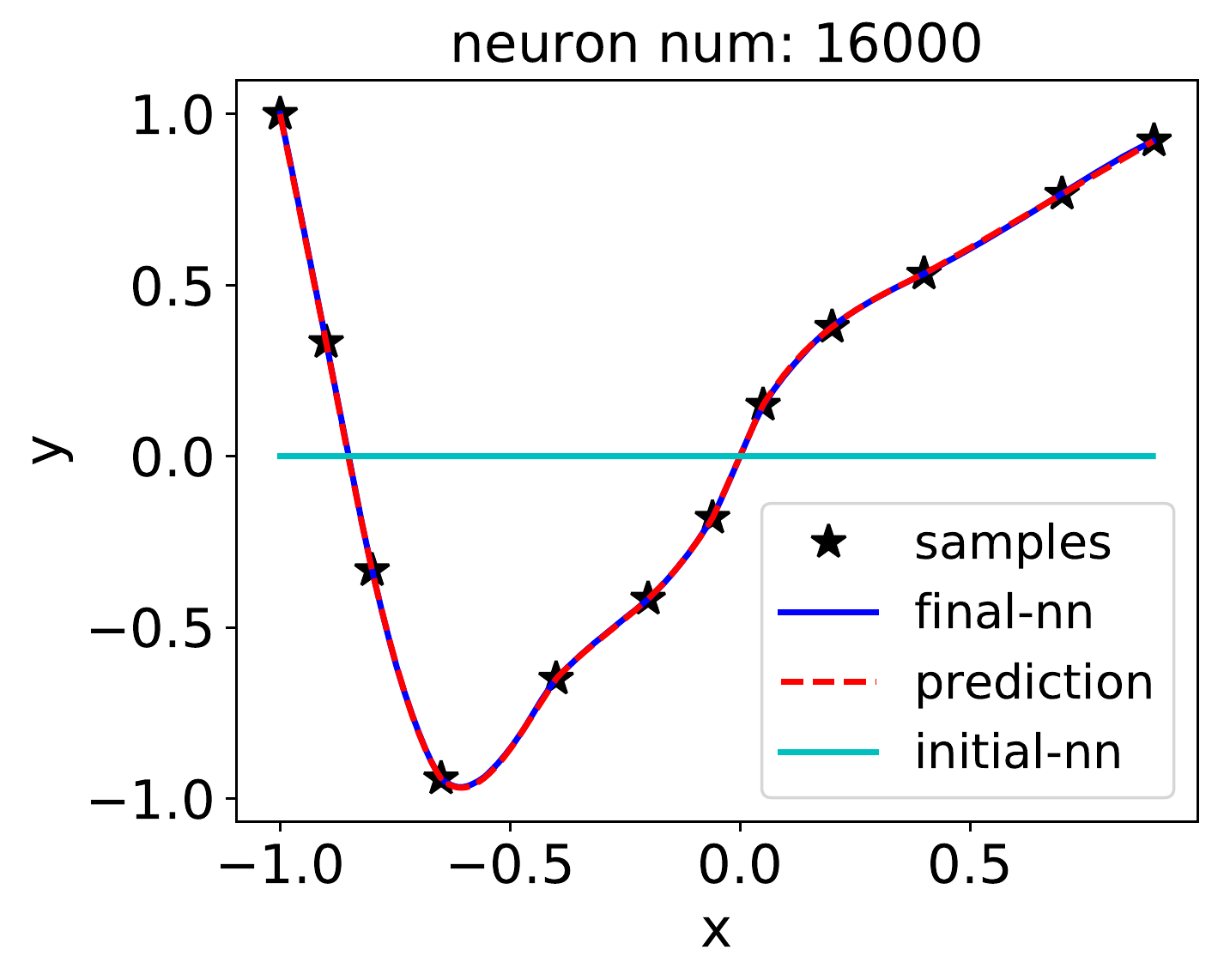} 
\par\end{centering}
}\subfloat[]{\begin{centering}
\includegraphics[scale=0.23]{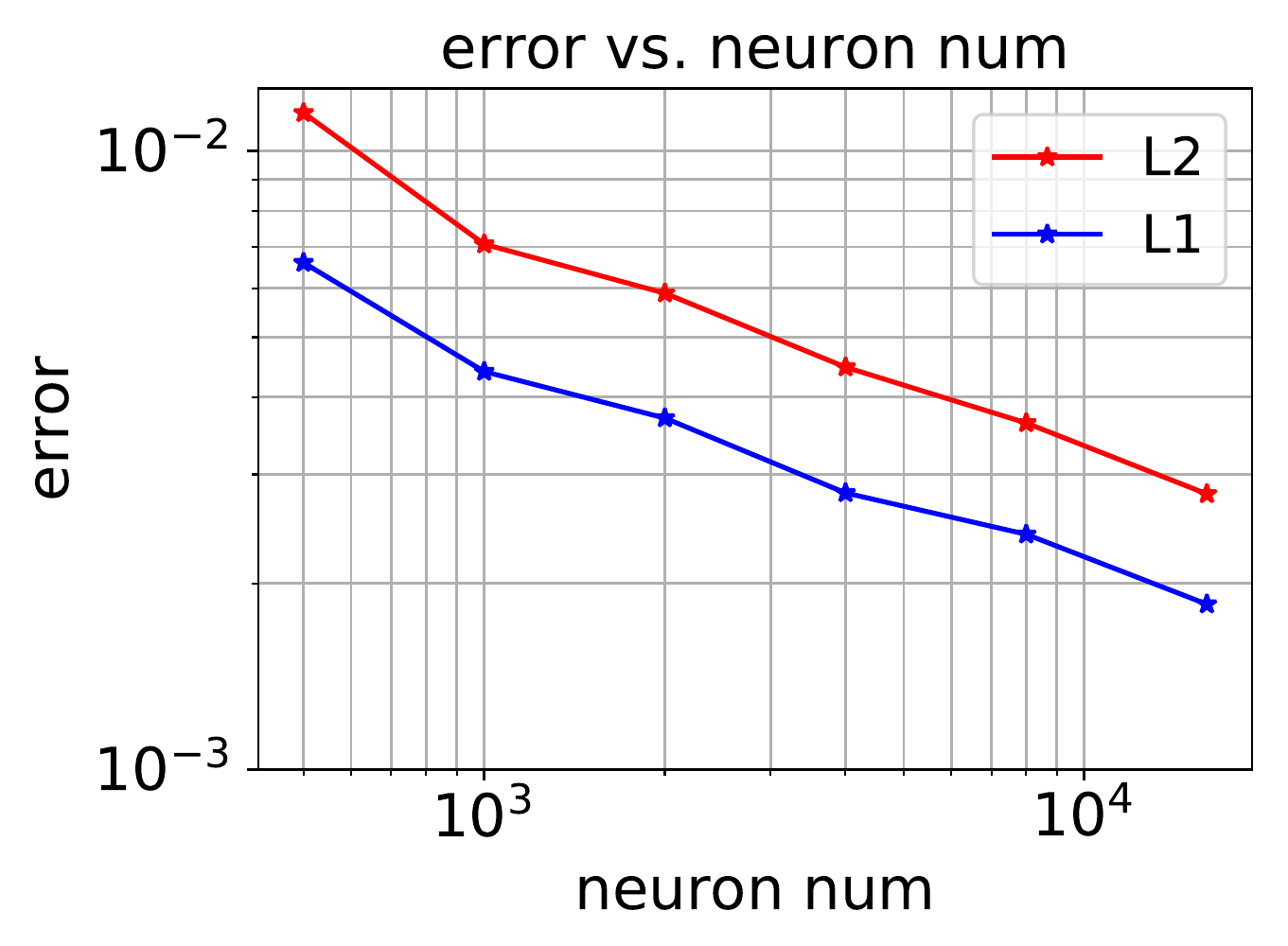} 
\par\end{centering}
}\subfloat[]{\begin{centering}
\includegraphics[scale=0.23]{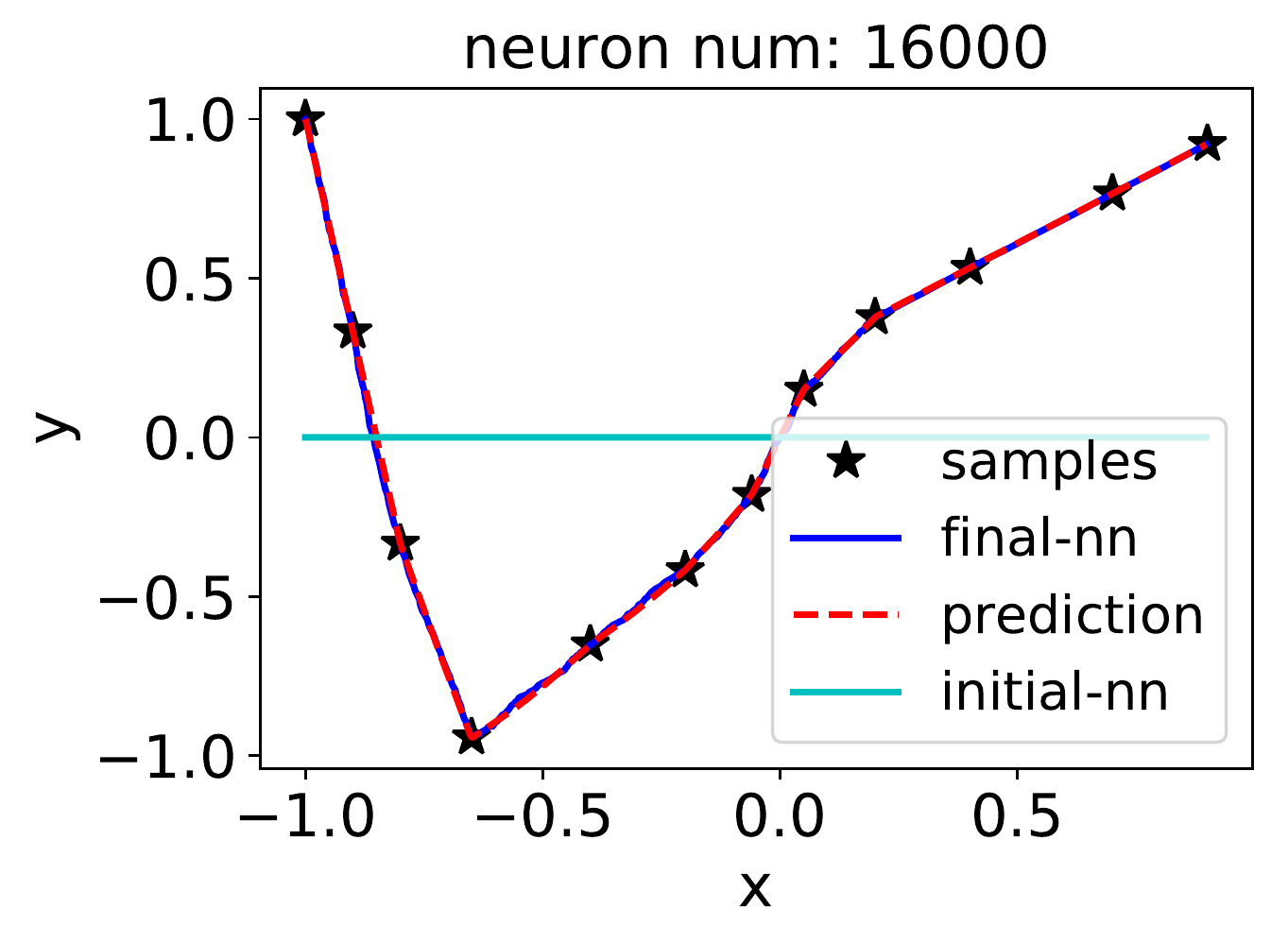} 
\par\end{centering}
}
\par\end{centering}
\caption{LFP model for 1-d training data. For (a, b, c), $1/|\xi|^{4}$ dominates. For (d), $1/|\xi|^{2}$ dominates. For (a, b, d), black stars: training samples; blue solid line: $800$ uniformly spaced samples; red dashed lines: solutions of the corresponding LFP models; cyan solid curves: zero initial outputs of NNs. (c) $L^{1}(h_{N},h_{\mathrm{LFP}})$
and $L^{2}(h_{N},h_{\mathrm{LFP}})$ (mean of 10 trials) vs. neuron number.\label{fig:2relu} }
\end{figure}
\par\end{center}

\begin{center}
\begin{figure}
\begin{centering}
\subfloat[]{\begin{centering}
\includegraphics[scale=0.23]{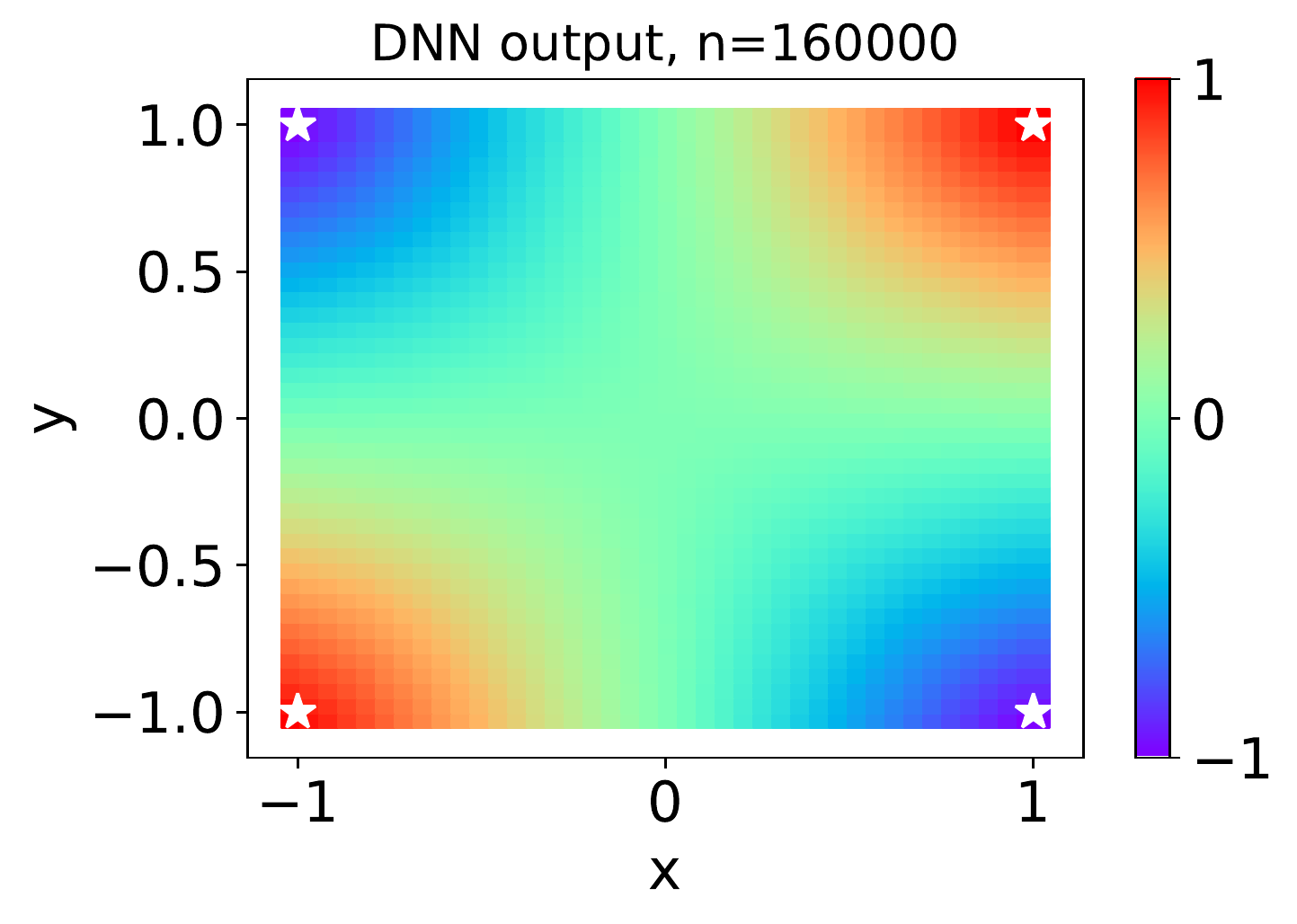} 
\par\end{centering}
}\subfloat[]{\begin{centering}
\includegraphics[scale=0.23]{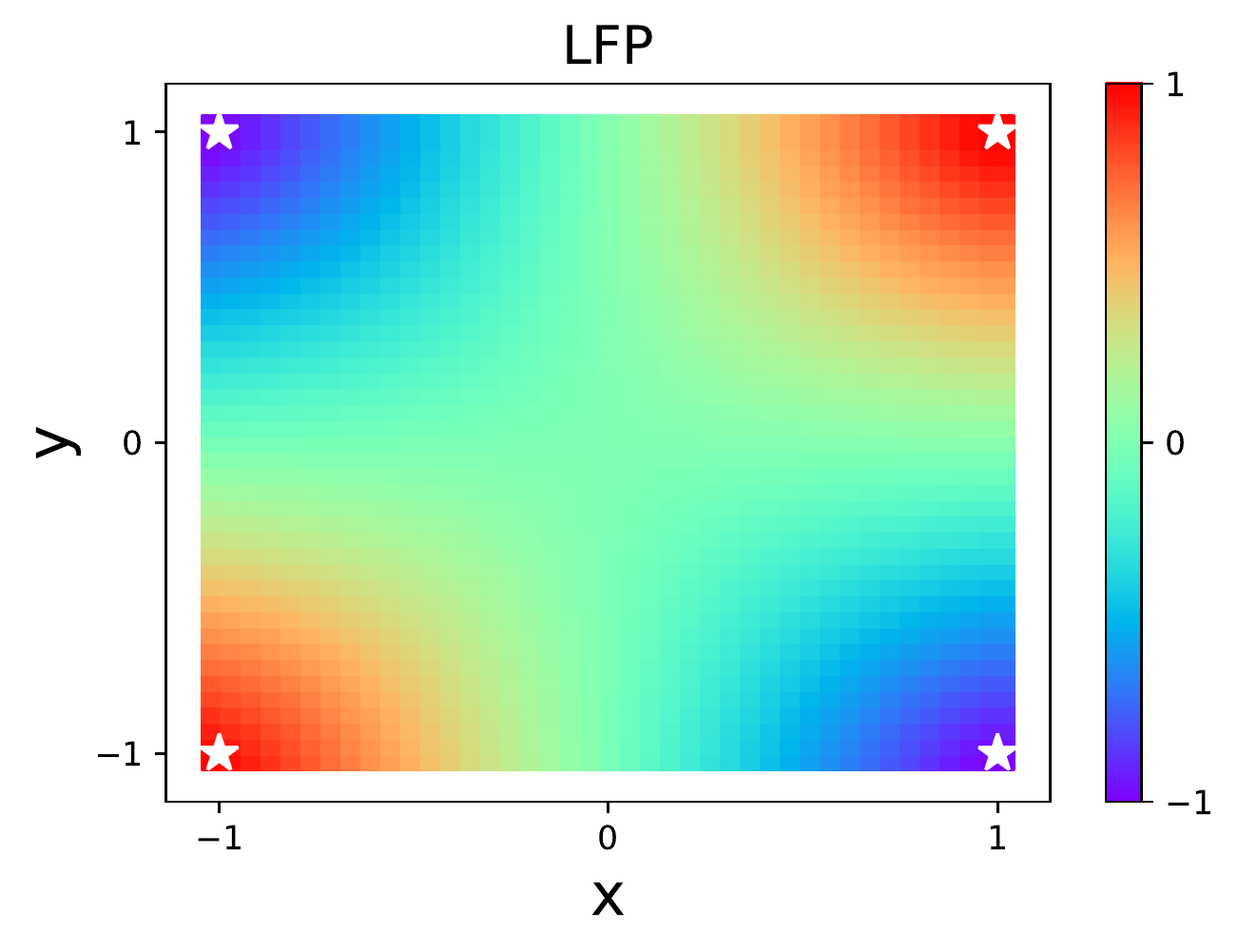} 
\par\end{centering}
}\subfloat[]{\begin{centering}
\includegraphics[scale=0.23]{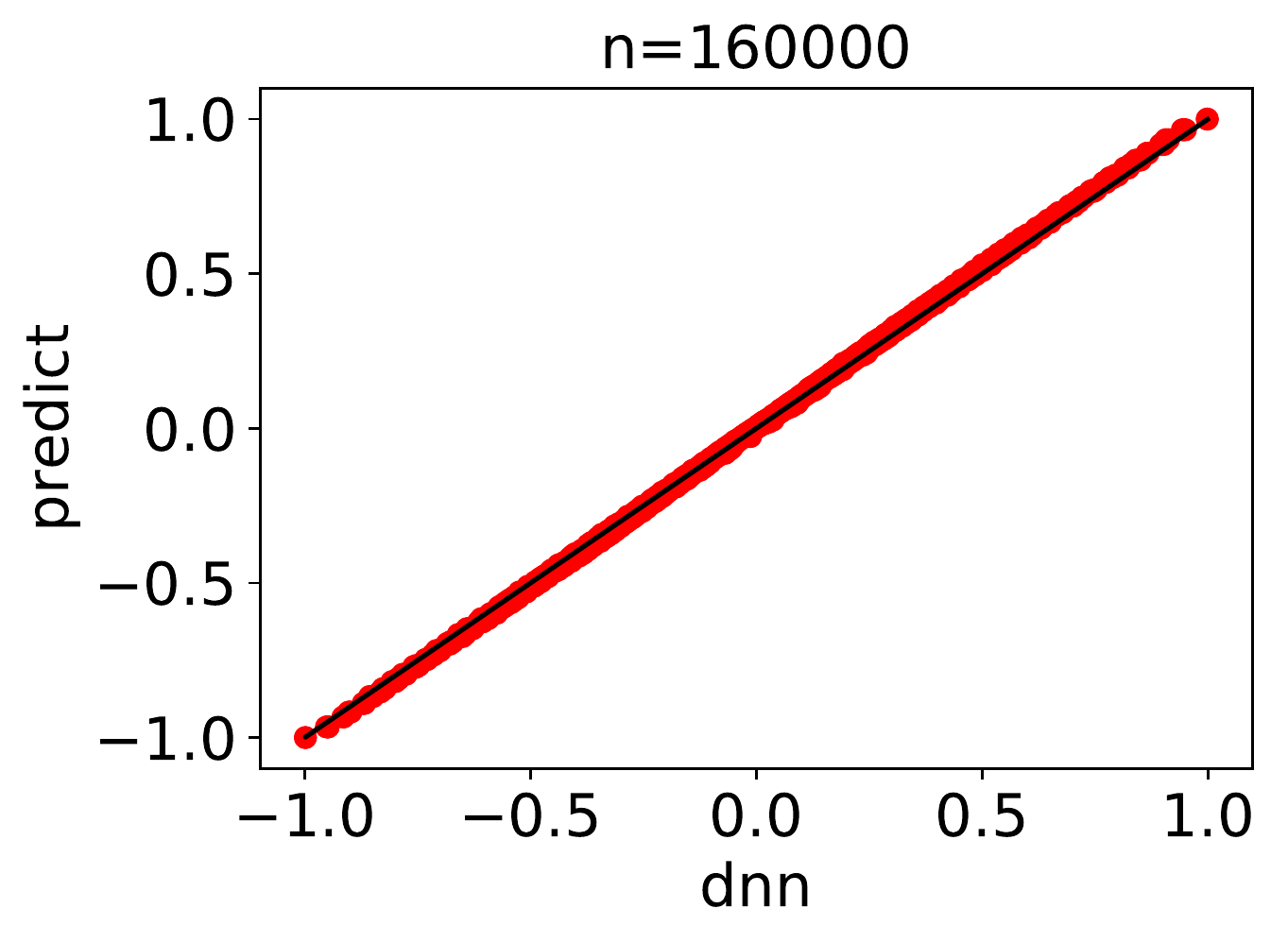} 
\par\end{centering}
}\subfloat[]{\begin{centering}
\includegraphics[scale=0.23]{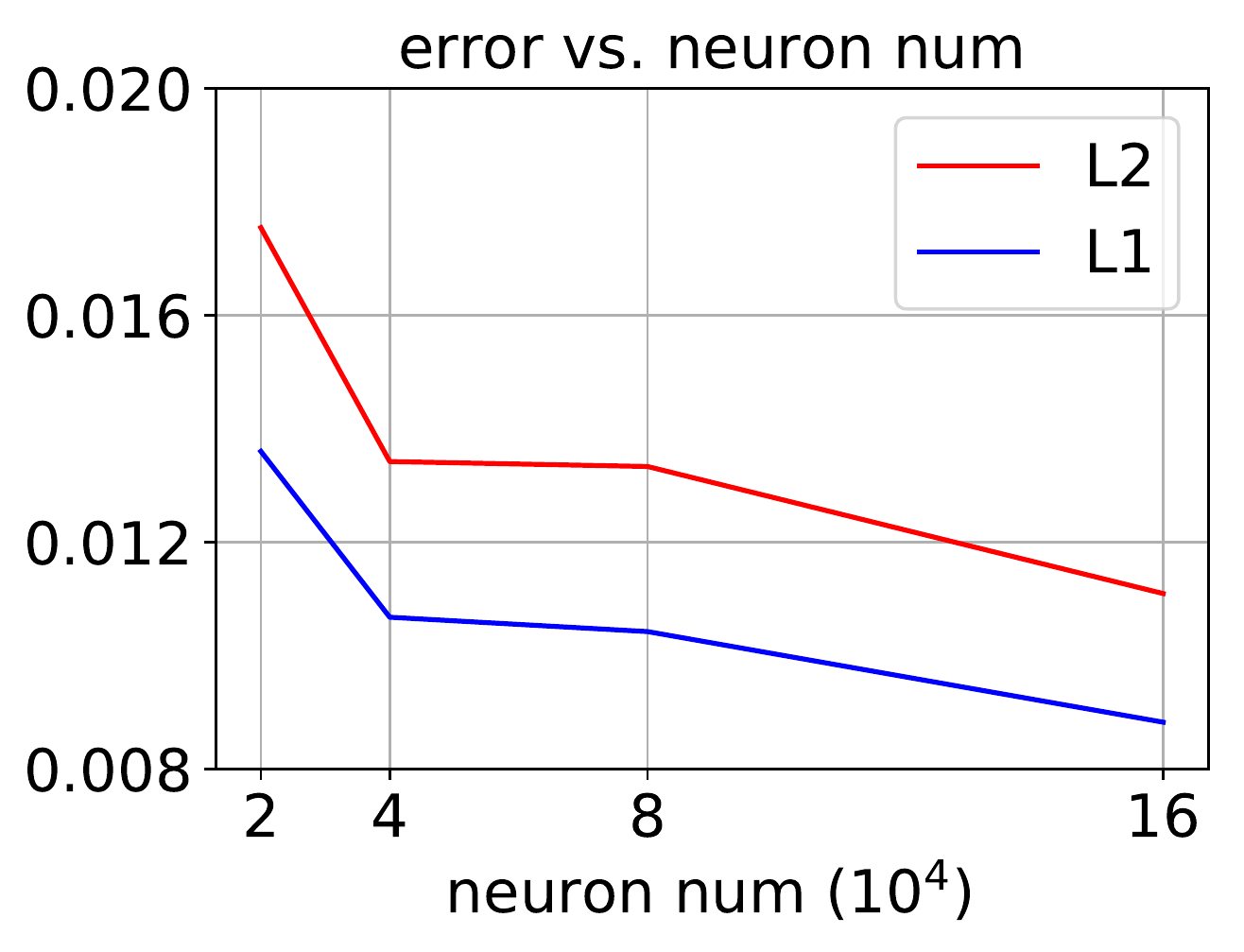} 
\par\end{centering}
}
\par\end{centering}
\caption{LFP model for 2-d training data of the XOR problem. (a) The final output
of the NN. (b) The solution of the corresponding LFP model. The training data are marked by white stars. (c) Each dot represents the final output of NN (abscissa)
vs. solution of the LFP model (ordinate) evaluated at one of the
$1600$ evenly spaced test points. The black line
indicates the identity function. (d) Decay of $L^{1}(h_{N},h_{\mathrm{LFP}})$
and $L^{2}(h_{N},h_{\mathrm{LFP}})$ (mean of 10 trials) vs. neuron number. \label{fig:2relu-1} }
\end{figure}
\par\end{center}

\subsection{Rationalization of the LFP dynamics}
Our starting point is the following \emph{linearized mean field residual dynamics}~(\cite{mei2018mean,mei2019mean}):
\begin{equation}
\textstyle{  \partial_t u(x, t) = -\int_{\R^d} K_{\theta_0}(x, z)u_p(z, t)\diff{z} \label{kernelds},
}\end{equation}
where $\theta_0$ denotes the initial parameters of the NN in the mean field kernel limit. The kernel is defined as
\begin{equation}\label{linear-kernelMain}
\textstyle{  K_{\theta_0}(x, z) = \int_{\R^{d+2}} \left[\nabla_{\theta}\sigma^*(x, \theta_0) \cdot \nabla_{\theta}\sigma^*(z, \theta_0)\right]\rho(w_0, r_0, l_0)\diff{w_0}\diff{r_0}\diff{l_0},
}\end{equation}
where $\sigma^*(\cdot, \theta_0) = w_0\sigma\left(r_0\cdot[\cdot] - \abs{r_0}l_0)\right)$, $\sigma$ is the ReLU function.
By applying the Fourier transform with respect to $x$ to both sides of Eq. (\ref{kernelds}), we can approximately derive the following frequency domain dynamics up to a time constant (see Appendix \ref{rationalization} for details), that is
\begin{equation}
 \partial_t\uhat(\xi, t) = -\left\langle\frac{\abs{r_0}^2+w_0^2}{\abs{\xi}^{d+3}} + \frac{4\pi^2\abs{r_0}^2 w_0^2}{\abs{\xi}^{d+1}}\right\rangle_{w_0, \abs{r_0}}\widehat{u_p}(\xi, t),
\end{equation}
where $\langle\cdot\rangle_{w_0, \abs{r_0}} = \int_{\R\times\R^+}\cdot \rho(w_0)\rho(\abs{r_0})\diff{w_0}\diff{\abs{r_0}}.$ Then Eq. (\ref{eq:relunnFP}) can be obtained by replacing $\langle\cdot\rangle_{w_0, \abs{r_0}}$ with the mean over $N$ hidden neurons.

\section{Explicitizing the implicit bias of the F-Principle \label{sec:Explicitizing-the-implicit}}

\subsection{An equivalent optimization problem to the gradient flow dynamics}

In our LFP model, the solution is implicitly regularized by a decaying coefficient for different frequencies of $\hat{u}$ throughout the training. For a quantitative analysis of this solution, we explicitize such an implicit dynamical regularization by a constrained optimization problem as follows.

First, we present a general theorem that the long-time limit solution of a gradient
flow dynamics is equivalent to the solution of a constrained optimization
problem. All proofs are in Appendix \ref{Sec:equiv}.

Let $H_1$ and $H_2$ be two seperable Hilbert spaces and $P: H_1\rightarrow H_2$ is a surjective linear operator, i.e., $\mathrm{Im} (P)=H_2$. Let $P^*: H_2\rightarrow H_1$ be the adjoint operator of $P$, defined by 
\begin{equation}
  \langle P u_1, u_2\rangle_{H_2}=\langle u_1, P^* u_2\rangle_{H_1},\quad\text{for all}\quad u_1\in H_1, u_2\in H_2.
\end{equation}
 
Given $g\in H_2$, we consider the following two problems.

(i)
The initial value problem
\begin{equation*}
\dfrac{\D u}{\D t}=P^*(g-Pu),\quad u(0)=u_{\rm ini}.
\end{equation*}
Since this equation is linear and with nonpositive eigenvalues on the right hand side, there exists a unique global-in-time solution $u(t)$ for all $t\in[0,+\infty)$ satisfying the initial condition. Moreover, the long-time limit $\lim_{t\rightarrow+\infty}u(t)$ exists and will be denoted as $u_\infty$.

(ii)
The minimization problem
\begin{equation*}
\min_{u -u_{\rm ini}\in H_1}\norm{u-u_{\rm ini}}_{H_1},\quad \text{s.t.}\quad Pu=g.
\end{equation*}
In the following, we will show that it has a unique minimizer which is denoted as $u_{\min}$. Now we present the following theorem of the equivalence relation.
\begin{thm}\label{MTthm..EquivalenceDynamicsMinimization}
  Suppose that $PP^*$ is surjective. The above Problems (i) and (ii) are equivalent in the sense that $u_\infty=u_{\min}$. 
  More precisely, we have
  \begin{equation}
    u_\infty=u_{\min}=P^*(PP^*)^{-1}(g-Pu_{\rm ini})+u_{\rm ini}.
  \end{equation}
\end{thm}

The following corollary is obtained directly from
Theorem \ref{MTthm..EquivalenceDynamicsMinimization}.

\begin{cor}\label{cor..EquivalencdHWFrequency} Let $\gamma: \R^{d}\rightarrow\R^+$ be a positive function and $h$ be a function in $L^2(\R^{d})$. The operator $\Gamma: L^2(\R^{d})\rightarrow L^2(\R^{d})$ is defined by $[\Gamma\hat{h}](\xi)=\gamma(\xi)\hat{h}(\xi)$, $\xi\in\R^{d}$.
  Define the Hilbert space $H_\Gamma:=\mathrm{Im}(\Gamma)$. 
  Let $X=(x_i)_{i=1}^M\in \R^{d\times M}$, $Y=(y_i)_{i=1}^M \in \R^{M}$ and $P: H_\Gamma\rightarrow \R^{M}$ be a surjective operator
  \begin{equation}
  P: \hat{h}\mapsto \left(\int_{\R^{d}}\hat{h}(\xi)\E^{2\pi\I x_i\cdot \xi}\diff{\xi}\right)_{i=1}^M=(h(x_i))_{i=1}^M.
  \end{equation}
  Then the following two problems are equivalent in the sense that $\hat{h}_\infty=\hat{h}_{\min}$.

  The initial value problem
  \begin{equation*}
  \dfrac{\D \hat{h}(\xi)}{\D t}=(\gamma(\xi))^2\sum_{i=1}^M(y_i\E^{-2\pi\I x_i\cdot\xi}-\hat{h}(\xi)*\E^{-2\pi\I x_i\cdot\xi}),\quad \hat{h}(0)=\hat{h}_{\rm ini}.
  \end{equation*}
The minimization problem
  \begin{equation*}
    \min_{\hat{h}-\hat{h}_{\rm ini}\in H_\Gamma}\int_{\R^{d}}(\gamma(\xi))^{-2}|\hat{h}(\xi)-\hat{h}_{\rm ini}(\xi)|^2\diff{\xi},\quad \rm{s.t.}\quad h(x_i)=y_i,\quad i=1,\cdots,M.
  \end{equation*}
\end{cor}

Note that in Appendix \ref{Sec:equiv}, we provide another version of Corollary
\ref{cor..EquivalencdHWFrequency} for the discretized frequency, which is considered in Section \ref{FPapriori}.

\subsection{Explicitizing the implicit bias for two-layer NNs \label{subsec:Explicit-regularization-of}}

By Corollary \ref{cor..EquivalencdHWFrequency}, we derive
the following constrained optimization problem explicitly minimizing
an FP-norm (see Section \ref{FPnorm}), whose solution is equivalent to that of the LFP model  (\ref{eq:relunnFP}),
that is, 
\begin{equation}
\min_{h-h_{\mathrm{ini}}\in F_{\gamma}}\int\left(\frac{\frac{1}{N}\sum_{i=1}^{N}\left(|r_{i}(0)|^{2}+w_{i}(0)^{2}\right)}{|\xi|^{d+3}}+\frac{4\pi^{2}\frac{1}{N}\sum_{i=1}^{N}\left(|r_{i}(0)|^{2}w_{i}(0)^{2}\right)}{|\xi|^{d+1}}\right)^{-1}|\hat{h}(\xi)-\hat{h}_{{\rm ini}}(\xi)|^{2}\mathrm{d}\xi,\label{eq: minFPnorm}
\end{equation}
subject to constraints
$h(x_{i})=y_{i}$ for $i=1,\cdots,M$. 
Note that the solutions of the LFP models in Figs. (\ref{fig:2relu}, \ref{fig:2relu-1})
are obtained by solving another form of this optimization problem (see Appendix \ref{sec:Numerical-solution-of}). This explicit penalty
indicates that the learning of DNN is biased towards functions with
more power at low frequencies (more precisely, functions of smaller FP-norm),
which is speculated in \citet{xu_training_2018,rahaman2018spectral,xu2019frequency}.
Next, we extend this special example to the case of a general weight function $(\gamma(\xi))^{-2}$ in the frequency domain. 

\section{FP-norm and an \emph{a priori} generalization error bound \label{FPapriori}}

The equivalent explicit optimization problem (\ref{eq: minFPnorm})
provides a way to analyze the generalization of sufficiently wide two-layer NNs. We begin with the definition of
an FP-norm, which naturally induces a FP-space containing all possible solutions of a target NN, whose
 Rademacher complexity can be controlled by the FP-norm of the
target function. Thus we obtain an \emph{a priori} estimate of the generalization error of NN by the theory of Rademacher complexity. Our \emph{a priori}
estimates follows the Monte Carlo
error rates with respect to the sample size. Importantly, Our estimate unravels how frequency components of the target function affect the generalization performance of DNNs. 

\subsection{FP-norm and FP-space\label{FPnorm}}

We denote $\Z^{d*}:=\Z^d\backslash\{0\}$.
Given a frequency weight function $\gamma: \Z^{d}\to \R^+$ or $\gamma: \Z^{d*}\to \R^+$ satisfying
\begin{equation}
  \norm{\gamma}_{\ell^2}=\left(\sum_{k\in\Z^{d}}(\gamma(k))^2\right)^{\frac{1}{2}}<+\infty\quad\text{or}\quad \norm{\gamma}_{\ell^2}=\left(\sum_{k\in\Z^{d*}}(\gamma(k))^2\right)^{\frac{1}{2}}<+\infty,
\end{equation}
we define the FP-norm for all function $h\in L^2(\Omega)$: 
\begin{equation}
  \norm{h}_{\gamma}:=\norm{\hat{h}}_{H_\Gamma}=\left(\sum_{k\in\Z^{d}}(\gamma(k))^{-2}|\hat{h}(k)|^{2}\right)^{\frac{1}{2}}.
\end{equation}
If $\gamma:\Z^{d*}\to\R^+$ is not defined at $0$, we set $(\gamma(0))^{-1}:=0$ in the above definition and $\norm{\cdot}_\gamma$ is only a semi-norm of $h$. Next we define the FP-space
\begin{equation}
  F_{\gamma}(\Omega)=\{h\in L^2(\Omega):\norm{h}_{\gamma}<\infty\}.
\end{equation}
Clearly, for any $\gamma$, the FP-space is a subspace of $L^2(\Omega)$. In addition, if $\gamma: k\mapsto |k|^{-m/2}$ for $k\in\Z^{d*}$, then functions in the FP-space with $\hat{h}(0)=\int_{\Omega} h(x) \diff{x}=0$ form the Sobolev space $H^m(\Omega)$. Note that in the case of DNN, according to the F-Principle, $(\gamma(k))^{-2}$
increases with the frequency. Thus, the contribution of high frequency
to the FP-norm is more significant than that of low frequency.

\subsection{\emph{a priori} generalization error bound }

The following lemma shows that the FP-norm closely relates to the
Rademacher complexity, which is defined as 
\begin{equation}
\tilde{R}(\mathcal{H})=\frac{1}{M}\mathbb{E}_{\e}\left[\sup_{h\in\mathcal{H}}\sum_{i=1}^{M}\e_{i}h(x_{i})\right].
\end{equation}
for the function space $\mathcal{H}$.

\begin{lem}
\label{Rad boundMT}  (i) For $\mathcal{H}=\{h:\norm{h}_{\gamma}\leq Q\}$ with $\gamma: \Z^{d}\to \R^+$, we have
  \begin{equation}
    \tilde{R}(\mathcal{H})\leq\frac{1}{\sqrt{M}}Q\norm{\gamma}_{\ell^2}.
  \end{equation}
  (ii) For $\mathcal{H}'=\{h:\norm{h}_{\gamma}\leq Q, \abs{\hat{h}(0)}\leq c_{0}\}$ with $\gamma: \Z^{d*}\to \R^+$ and $\gamma^{-1}(0):=0$, we have
  \begin{equation}
    \tilde{R}(\mathcal{H}')\leq \frac{c_{0}}{\sqrt{M}}+\frac{1}{\sqrt{M}}Q\norm{\gamma}_{\ell^2}.
  \end{equation}
\end{lem}

By Lemma \ref{FPnorm Bound} in Appendix \ref{proofapriori}, the FP-norm of the solution for the optimization problem in Corollary \ref{cor..EquivalencdHWFrequency} is bounded
by $Q=\norm{f-h_{\mathrm{ini}}}_{\gamma}$. Then, we obtain the following estimate of the generalization error bound.
\begin{thm}
\label{thm:priorierror}Suppose that the real-valued target function $f\in F_\gamma(\Omega)$, the training dataset $\{x_{i}; y_i\}_{i=1}^{M}$ satisfies $y_i=f(x_{i})$, $i=1,\cdots,M$, and $h_{M}$ is the solution of the regularized model
  \begin{equation}
    \min_{h-h_{\rm ini}\in F_
    \gamma(\Omega)} \norm{h-h_{{\rm ini}}}_{\gamma},\quad\text{s.t.}\quad h(x_i)=y_i,\quad i=1,\cdots,M.
  \end{equation}
  Then we have

  (i) given $\gamma: \Z^{d}\to \R^+$, 
  for any $\delta\in(0,1)$, with probability at least $1-\delta$ over
  the random training samples, the population risk has the bound
  \begin{equation}
    L(h_{M})
    \leq \norm{f-h_{{\rm ini}}}_{\gamma}\norm{\gamma}_{\ell^2}
    \left(\frac{2}{\sqrt{M}}+4\sqrt{\frac{2\log(4/\delta)}{M}}\right).
  \end{equation}

  (ii) given $\gamma: \Z^{d*}\to \R^+$ with $\gamma(0)^{-1}:=0$, for any $\delta\in(0,1)$,
  with probability at least $1-\delta$ over the random training samples,
  the population risk has the bound
  \begin{equation}
    L(h_{M})
    \leq \left(\norm{f-h_{\rm ini}}_{\infty}+2\norm{f-h_{{\rm ini}}}_{\gamma}\norm{\gamma}_{\ell^2}
    \right)
    \left(\frac{2}{\sqrt{M}}+4\sqrt{\frac{2\log(4/\delta)}{M}}\right).
  \end{equation}
\end{thm}

\begin{rem}
  By the assumption in the theorem, the target function $f$ belongs to $F_\gamma(\Omega)$ which is a subspace of $L^2(\Omega)$. In most applications, $f$ is also a continuous function. In any case, $f$ can be well-approximated by a large neural network due to universal approximation theory \citet{cybenko1989approximation}.
\end{rem}

Our a priori generalization error bound in Theorem. \ref{thm:priorierror}
is large if the target function possesses significant high frequency
components. Thus, it explains the failure of DNNs in generalization for learning the
parity function \citep{shalev2017failures}, whose power concentrates at high
frequencies. In the following, We use experiments to illustrate that, as predicted by our a priori generalization error bound, larger FP-norm of the target function indicates a larger generalization error.

\subsection{Experiment}

In this section, we train a ReLU-NN of width 1-5000-1 to fit 20
uniform samples of $f(x)=\sin(2\pi vx)$ on $[0,1]$ until the training
MSE loss is smaller than $10^{-6}$, where $v$ is the frequency. We then use 500 uniform samples
to test the NN. The FP-norm of the target function is computed by 

\begin{equation}
\norm{f}_{\gamma}=\left(\sum_{\xi\in\Z^{d*}}\left(\frac{\frac{1}{N}\sum_{i=1}^{N}\left(|r_{i}(0)|^{2}+w_{i}(0)^{2}\right)}{|\xi|^{d+3}}+\frac{4\pi^{2}\frac{1}{N}\sum_{i=1}^{N}\left(|r_{i}(0)|^{2}w_{i}(0)^{2}\right)}{|\xi|^{d+1}}\right)\abs{\hat{f}(\xi))}^{2}\right)^{1/2},\label{eq:fpdft}
\end{equation}
where $\hat{f}(\xi)$ is computed by the discrete Fourier transform
of $f(x)$. As shown in Fig. \ref{fig:fpnorm}, a larger FP-norm of the target function is related to a larger test error. 
\begin{center}
\begin{figure}
\begin{centering}
\includegraphics[scale=0.3]{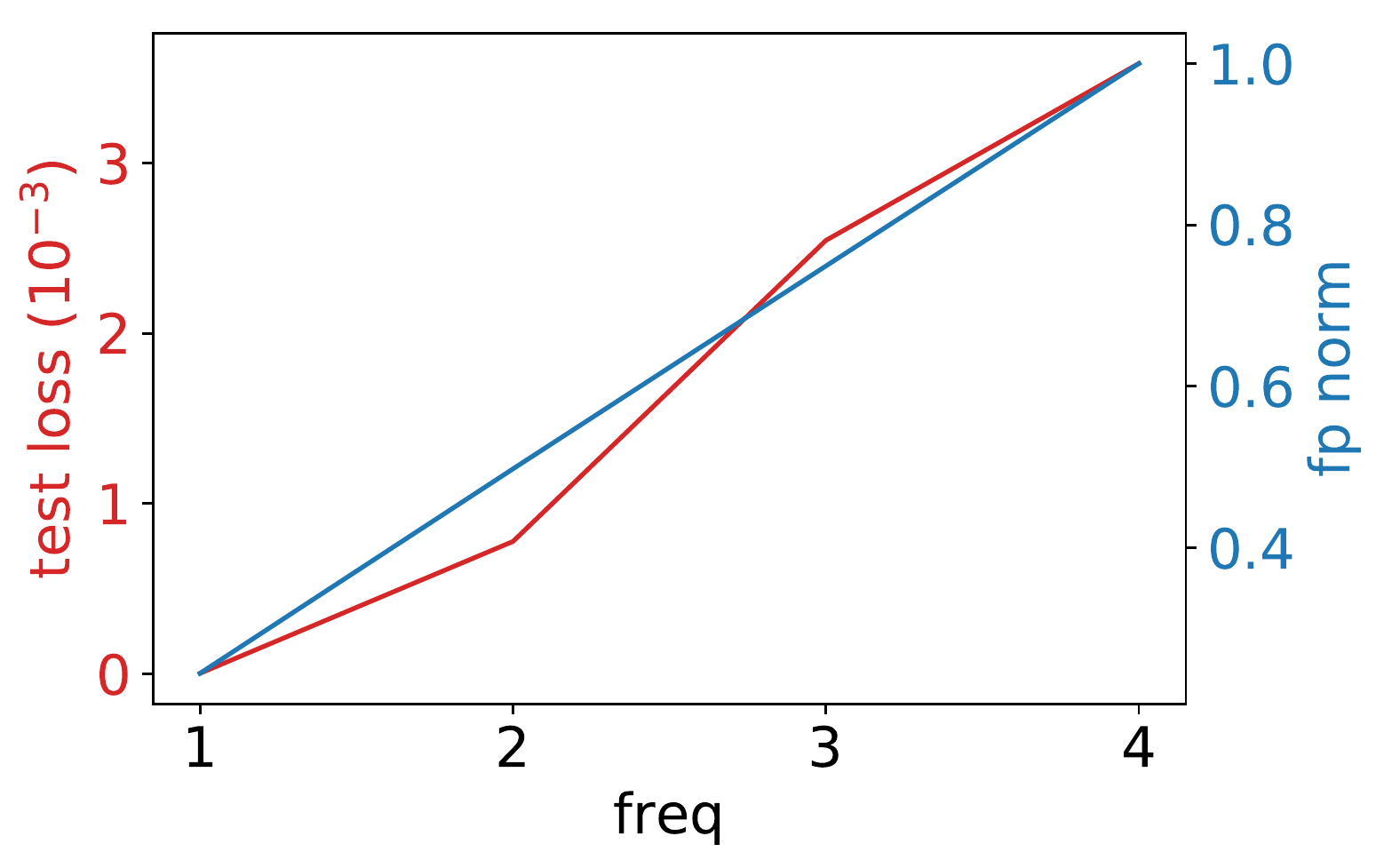} 
\par\end{centering}
\caption{normalized FP-norm and test loss are plotted as a function of frequency $v$ of the target function $\sin(2\pi vx)$.\label{fig:fpnorm} }
\end{figure}
\par\end{center}

\section{Discussion}

In this work, inspired by the F-Principle, we propose an effective
LFP model for NNs --- a model quantitatively well predicts the output
of wide two-layer ReLU NNs and is theoretically rationalized by their
training dynamics in an extremely over-parameterized regime. We explicitize
the implicit bias of the F-Principle by a constrained optimization
problem equivalent to the LFP model. This explicitization leads to
an \emph{a priori} estimate of the generalization error bound, which
depends on the FP-norm of the target function. Note
that, our LFP model based on the ReLU transfer function
can be naturally extended to other transfer functions following a
similar construction process.

As a candidate of an effective model of DNNs, the LFP model advances
our qualitative/empirical understandings of the F-Principle to a quantitative
level. i) With ASI trick \citep{zhang_type_2019} offsetting the initial
DNN output to zero, the LFP model indicates that the F-Principle also
holds for DNNs initialized with large weights. Therefore, ``initialized
with small parameters'' \citep{xu_training_2018,xu2019frequency}
is not a necessary condition for the F-Principle. ii) Based on
the qualitative behavior of F-Principle, previous works \citep{xu_training_2018,xu2019frequency,rahaman2018spectral}
speculate that ``DNNs prefer to learn the training data by a low
frequency function''. With an equivalent optimization problem explicitizing
the F-Principle, the LFP model quantifies this speculation. 

Our \emph{a priori} generalization error bound increases as the FP-norm
of the target function increases. This explains several important
phenomena. First, DNNs fail to generalize well for the parity function
\citep{shalev2017failures}. \citet{xu2019frequency} shows that this
is due to the inconsistency between the high frequency dominant property
of the parity function and the low frequency preference of DNNs. In
this work, by our \emph{a priori} generalization error bound, the dominant
high frequency of the parity function quantitatively results in a
large FP-norm, thus, a large generalization error. Second, because
randomly labeled data possesses large high frequency components, which
induces a large FP-norm of any function well matches the training
data and test data, we expect a very large generalization error, e.g., no generalization,
as observed in experiments. Intuitively, our estimate indicates
good generalization of  Ns forNwell-structured low-frequency
dominant real dataset as well as bad generalization of NNs for randomly labeled data, thus providing
insight into the well known puzzle of generalization of DNNs \citep{zhang2016understanding}.

The F-Principle, a widely observed implicit bias of DNNs, is also
a natural bias for human. Empirically, when a human see several
points of training data, without a specific prior, one tends to interpolate
these points by a low frequency dominant function. Therefore, the success of
DNN may partly result from its adoption of a similar interpolation
bias as human's. In general, there could be multiple types of implicit biases underlying the training dynamics of a DNN. Inspired
by the LFP model, discovering and explicitizing these implicit biases
could be a key step towards a thorough quantitative understanding
of deep learning.

\section*{Acknowledgments}

The authors want to thank Prof. Weinan E for helpful discussions. ZX, YZ are supported by the NYU Abu Dhabi Institute G1301. 

\newpage{}

\subsection*{\bibliographystyle{agsm}
\bibliography{DLRef}
}

\newpage{}

\part*{Appendix}

\section{Rationalization of the LFP dynamics \label{rationalization}}
In this section, we derive the dynamics of each frequency component
of the loss function when a two-layer ReLU-NN is used to fit a $d$-dimensional function. Under mild assumption, we can clearly see that a lower frequency component has a faster convergence speed.

The starting point is the following \emph{linearized mean field residual dynamics}~\cite{mei2018mean,mei2019mean}:
\begin{equation}
  \partial_t u(x, t) = -\int_{\R^d} K_{\theta_0}(x, z)u_p(z, t)\diff{z},
\end{equation}
here $\theta_0$ denotes the initial parameters of the training dynamics, considering the linear kernel regime. The kernel is defined as
\begin{equation}\label{linear-kernel}
  K_{\theta_0}(x, z) = \int_{\R^{d+2}} \left[\nabla_{\theta}\sigma^*(x, \theta_0) \cdot \nabla_{\theta}\sigma^*(z, \theta_0)\right]\rho(w_0, r_0, l_0)\diff{w_0}\diff{r_0}\diff{l_0},
\end{equation}
with $\sigma^*(\cdot, \theta_0) = w_0\sigma\left(r_0\cdot[\cdot] - \abs{r_0}l_0)\right)$ and $\sigma$ is the ReLU function. For simplicity, we assume the parameters are isotropic and $\rho(l_0)=C_{l_0}$, $\rho(e_{r_0}) = C_{e_{r_0}}$ to be specified later that is,
\begin{equation}
    \rho(w_0, r_0, l_0)\diff{w_0}\diff{r_0}\diff{l_0} = C_{e_{r_0}}C_{l_0}\rho(w_0)\rho(\abs{r_0})\diff{w_0}\diff{\abs{r_0}}\diff{e_{r_0}}\diff{l_0} := \rho(\diff{\theta_0}),
\end{equation}
where $e_{r_0}$ is the unit vector of $r_0$. Then gradient of $\sigma^*$ with respect to the parameters is
\begin{equation}\label{grads}
  \nabla_{\theta}\sigma^*(\cdot, \theta_0)
  = \begin{pmatrix}
  \sigma(r_0\cdot [\cdot] - \abs{r_0}l_0) \\
  w_0(\cdot - l_0 e_{r_0})\sigma'(r_0\cdot [\cdot] - \abs{r_0}l_0)  \\
  - w_0 \abs{r_0} \sigma'(r_0\cdot [\cdot] - \abs{r_0}l_0)
  \end{pmatrix}.
\end{equation}

For any given vector $r_0$, we can decompose the gradient with respect to $r_0$, that is, second row in the above gradient vector, into two directions: parallel and perpendicular to $r_0$, i.e.,
\begin{align}
  \nabla_{r_0}\sigma^*(\cdot, \theta_0) 
  & = \frac{r_0\cdot\nabla_{r_0}\sigma^*(\cdot, \theta_0)}{\abs{r_0}^2}r_0 + \left[\nabla_{r_0}\sigma^*(\cdot, \theta_0) - \frac{r_0\cdot\nabla_{r_0}\sigma^*(\cdot, \theta_0)}{\abs{r_0}^2}r_0\right] \\
  & = w_0(r_0\cdot [\cdot] - \abs{r_0} l_0)\sigma'(r_0\cdot [\cdot] - \abs{r_0} l_0)\frac{r_0}{\abs{r_0}^2} + w_0 [\cdot]_{\perp r_0}\sigma'(r_0\cdot [\cdot] - \abs{r_0} l_0) \\
  & = w_0\sigma(r_0\cdot [\cdot] - \abs{r_0}l_0)\frac{
    r_0}{\abs{r_0}^2} + w_0 [\cdot]_{\perp r_0}\sigma'(r_0\cdot [\cdot] - \abs{r_0}l_0),
\end{align}
where $[\cdot]_{\perp r_0} = [\cdot] - \frac{[\cdot]\cdot r_0}{\abs{r_0}^2}r_0$ and $[\cdot]\sigma'(\cdot) = \sigma(\cdot)$ due to the property of the ReLU function.
Thus, the gradients~\eqref{grads} can be written into two parts
\begin{align}
  \nabla_{\theta}\sigma^*(\cdot, \theta_0)
  & =
  \begin{pmatrix}
    \sigma(r_0\cdot [\cdot] - \abs{r_0}l_0) \\
    \frac{w_0}{\abs{r_0}^2}\sigma(r_0\cdot [\cdot] - \abs{r_0}l_0)r_0 \\
    - w_0 \abs{r_0}\sigma'(r_0\cdot [\cdot] - \abs{r_0}l_0)
  \end{pmatrix}
  +
  \begin{pmatrix}
    0 \\
    w_0\sigma'(r_0\cdot [\cdot] - \abs{r_0}l_0)[\cdot]_{\perp r_0}  \\
    0
  \end{pmatrix} \\
  & := A(r_0\cdot [\cdot] - \abs{r_0}l_0) + B.
\end{align}
Then the kernel~\eqref{linear-kernel} can be split into two parts,
\begin{align}
  K_{\theta_0}(x, z) = & {} \int_{\R^{d+2}} A(r_0\cdot x - \abs{r_0}l_0)\cdot A(r_0\cdot z - \abs{r_0}l_0)\rho(\diff{\theta_0})+ \int_{\R^{d+2}} B\cdot B \rho(\diff{\theta_0}) \\
  =: & {} K + G,
\end{align}
In the following computation, we will drop the term $G$. Since this term corresponds to the direction perpendicular to $r_0$, it is not very easy to check by numerical experiments and we only consider the dynamics of the parallel part.

Taking Fourier transform with respect to $x$,
\begin{align}
  \partial_t \hat{u}(\xi, t) & = - \int_{\R^d} \hat{K}_{\theta_0}(\xi, z)u_p(z,t)\diff{z} \\
  & = - \int_{\R^{2d+2}} \F[A(r_0\cdot [\cdot] - \abs{r_0} l_0)](\xi)\cdot A(r_0\cdot z - \abs{r_0} l_0)u_p(z,t)\rho(\diff{\theta_0})\diff{z}.
\end{align}
Notice for any function $f$ defined on $\R$, the following results hold:
\begin{align}\label{FT1}
  \F[f(r_0\cdot [\cdot] - \abs{r_0} l_0)](\xi) 
  & = \int_{\R^d} f(r_0\cdot x - \abs{r_0} l_0)\E^{-2\pi\I\xi\cdot x}\diff{x} \\
  & = \int_{\R^d} f(r_0\cdot x - r_0 l_0)\E^{-2\pi\I\xi\cdot(x - l_0 e_{r_0})} \E^{-2\pi\I\xi\cdot e_{r_0} l_0}\diff{x} \\
  & = \F[f(r_0\cdot [\cdot])](\xi)\E^{-2\pi\I\xi\cdot e_{r_0}l_0},
\end{align}
and
\begin{align}\label{FT2}
    \F[f(r_0\cdot [\cdot])](\xi) & = \int_{\R^d} f(r_0 \cdot x) \E^{-2\pi\I\xi\cdot x}\diff{x} \\
    & = \int_{\R^d} f(r_0\cdot x)\E^{-2\pi\I\xi\cdot x_{\parallel r_0}}\E^{-2\pi\xi\cdot x_{\perp r_0}}\diff{x_{\parallel r_0}}\diff{x_{\perp r_0}} \\
    & = \left(\int_{\R} f(r_0\cdot x_{\parallel r_0})\E^{-2\pi\I\xi\cdot x_{\parallel r_0}}\diff{x_{\parallel r_0}}\right)\left(\int_{\R^{d-1}} e^{-2\pi\I\xi\cdot x_{\perp r_0}} \diff{x_{\perp r_0}} \right) \\
    & = \left(\int_{\R} f(\abs{r_0}y)\E^{-2\pi\I\xi\cdot e_{r_0}y}\diff{y}\right)\delta(\xi_{\perp r_0}) \\
    & = \frac{1}{\abs{r_0}}\F[f]\left(\frac{\xi\cdot e_{r_0}}{\abs{r_0}}\right)\delta(\xi_{\perp r_0}),
\end{align}
where $x_{\parallel r_0} = \frac{r_0\cdot x}{\abs{r_0}^2}r_0$, $x_{\perp r_0} = x - x_{\parallel r_0}$, $\xi_{\perp r_0} = \xi - \frac{\xi\cdot r_0}{\abs{r_0}^2}r_0$ and $y = e_{r_0}\cdot x$. The above delta function is defined as, for any function $f(e_{r_0})$
\begin{equation}\label{delta_e}
  \int_{\mathbb{S}^{d-1}}\delta(\xi_{\perp r_0})f(e_{r_0})\rho(e_{r_0})\diff{e_{r_0}} = C_{e_{r_0}}\abs{\xi}^{-(d-1)}f(e_{\xi}),
\end{equation}
where $C_{e_{r_0}} = \frac{\Gamma(d/2)}{2\pi^{d/2}}$. Combining the above results, one obtains
\begin{align}
  \F[A(r_0\cdot [\cdot] - \abs{r_0} l_0)](\xi)
  & = \F\left[\begin{pmatrix}
   \sigma(r_0\cdot [\cdot] -\abs{r_0}l_0)] \\
    w_0\sigma(r_0\cdot [\cdot] - \abs{r_0}l_0)\frac{
    r_0}{\abs{r_0}^2}  \\
    -w_0 |r_0| \sigma'(r_0\cdot [\cdot] - \abs{r_0}l_0)
  \end{pmatrix}\right](\xi) \\
  & = -\frac{1}{4\pi^2}\begin{pmatrix}
    1 \\
    \frac{w_0 r_0}{\abs{r_0}^2}\\
    2\pi\I w_0(\xi\cdot e_{r_0})
  \end{pmatrix}\frac{|r_0|}{(\xi\cdot e_{r_0})^2}\E^{-2\pi\I\xi\cdot e_{r_0}l_0}\delta(\xi_{\perp r_0}).
\end{align}
So
\begin{multline}
  \partial_t \uhat(\xi, t) = \frac{1}{4\pi^2}\int_{\R^{2d+2}}\delta(\xi_{\perp r_0})\frac{\abs{r_0}}{(\xi\cdot e_{r_0})^2}\Bigg[\left(1+\frac{w^2_0}{\abs{r_0}^2}\right)\sigma(r_0\cdot z - \abs{r_0}l_0) \\
  - 2\pi\I \abs{r_0}w^2_0(\xi\cdot e_{r_0})\sigma'(r_0\cdot z - \abs{r_0}l_0)\Bigg]\E^{-2\pi\I\xi\cdot e_{r_0}l_0}u_p(z,t)\rho(\diff{\theta_0})\diff{z},
\end{multline}
then we first integrate the variable $e_{r_0}$ using~\eqref{delta_e} to get
\begin{align}
  &\partial_t \uhat(\xi, t) = \frac{C_{e_{r_0}}C_{l_0}}{4\pi^2\abs{\xi}^{d+1}}\int_{\R^{d+2}\times\R^+}\Bigg[\left(1+\frac{w^2_0}{\abs{r_0}^2}\right)\abs{r_0}\sigma(\abs{r_0}e_{\xi}\cdot z - \abs{r_0}l_0) \nonumber\\
  & \quad - 2\pi\I w^2_0\abs{r_0}^2\abs{\xi}\sigma'(\abs{r_0}e_{\xi}\cdot z - \abs{r_0}l_0)\Bigg]\E^{-2\pi\I\abs{\xi}l_0}u_p(z,t)\rho(w_0)\rho(\abs{r_0})\diff{w_0}\diff{\abs{r_0}}\diff{l_0}\diff{z} \\
  & = \frac{C_{e_{r_0}}C_{l_0}}{4\pi^2\abs{\xi}^{d+1}}\int_{\R^{d+2}\times\R^+}\Bigg[\left(1+\frac{w^2_0}{\abs{r_0}^2}\right)\abs{r_0}\sigma(\abs{r_0}e_{\xi}\cdot z - \abs{r_0}l_0)\E^{-2\pi\I\abs{\xi}(l_0 - e_{\xi}\cdot z)} \nonumber\\
  &\quad  - 2\pi\I w^2_0\abs{r_0}^2\abs{\xi}\sigma'(\abs{r_0}e_{\xi}\cdot z - \abs{r_0}l_0)\E^{-2\pi\I\abs{\xi}(l_0 - e_{\xi}\cdot z)}\Bigg] u_p(z,t)\E^{-2\pi\I\xi\cdot z}\rho(w_0)\rho(\abs{r_0}) \diff{w_0}\diff{\abs{r_0}}\diff{l_0}\diff{z} \\
  & = \frac{C_{e_{r_0}}C_{l_0}}{4\pi^2\abs{\xi}^{d+1}}\int_{\R\times\R^+}\big[(\abs{r_0}^2+ w_0^2)\overline{\F[\sigma](\abs{\xi})} \nonumber\\
  & \quad\quad\quad\quad\quad\quad - 2\pi\I w_0^2 \abs{r_0}^2\abs{\xi}\overline{\F[\sigma'(\abs{r_0} \cdot)](\abs{\xi})}\big] \widehat{u_p}(\xi, t)\rho(w_0)\rho(\abs{r_0})\diff{w_0}\diff{\abs{r_0}}.
\end{align}
Notice that
\begin{align}
  & \F[\sigma](\abs{\xi}) = - \frac{1}{4\pi^2\abs{\xi}^2}, \\
  & \F[\sigma'(\abs{r_0}\cdot)](\abs{\xi}) = \frac{1}{2\pi\I\abs{\xi}},
\end{align}
thus
\begin{equation}
  \partial_t\uhat(\xi, t) = -\frac{C_{e_{r_0}}C_{l_0}}{16\pi^4}\left\langle\frac{\abs{r_0}^2+w_0^2}{\abs{\xi}^{d+3}} + \frac{4\pi^2\abs{r_0}^2 w_0^2}{\abs{\xi}^{d+1}}\right\rangle_{w_0, \abs{r_0}}\widehat{u_p}(\xi, t),
\end{equation}
where
\begin{equation}
  \langle\cdot\rangle_{w_0, \abs{r_0}} = \int_{\R\times\R^+}\cdot \rho(w_0)\rho(\abs{r_0})\diff{w_0}\diff{\abs{r_0}}.
\end{equation}
Finally, by choosing a suitable time scale the above dynamics can be written as
\begin{equation}
  \partial_t\uhat(\xi, t) = -\left\langle\frac{\abs{r_0}^2+w_0^2}{\abs{\xi}^{d+3}} + \frac{4\pi^2\abs{r_0}^2 w_0^2}{\abs{\xi}^{d+1}}\right\rangle_{w_0, \abs{r_0}}\widehat{u_p}(\xi, t).
\end{equation}

\section{Proofs of the equivalence theorems}\label{Sec:equiv}

Let $H_1$ and $H_2$ be two seperable Hilbert spaces and $P: H_1\rightarrow H_2$ is a surjective linear operator, i.e., $\mathrm{Im} (P)=H_2$. Let $P^*: H_2\rightarrow H_1$ be the adjoint operator of $P$, defined by 
\begin{equation}
  \langle P u_1, u_2\rangle_{H_2}=\langle u_1, P^* u_2\rangle_{H_1},\quad\text{for all}\quad u_1\in H_1, u_2\in H_2.
\end{equation}

\begin{lem}\label{lem..spectrum.positive}
  Suppose that $H_1$ and $H_2$ are two seperable Hilbert spaces and $P:H_1\rightarrow H_2$ and $P^*:H_2\rightarrow H_1$ is the adjoint of $P$. Then all eigenvalues of $P^*P$ and $PP^*$ are non-negative. Moreover, they have the same positive spectrum. If in particular, we assume that the operator $PP^*$ is surjective, then the operator $PP^*$ is invertible.
\end{lem}
\begin{proof}
  We consider the eigenvalue problem $P^*Pu_1=\lambda u_1$. Taking inner product with $u_1$, we have $\langle u_1,P^*Pu_1\rangle_{H_1}=\lambda\norm{u_1}^2_{H_1}$. Note that the left hand side is $\norm{Pu_1}^2_{H_2}$ which is non-negative. Thus $\lambda\geq 0$. Similarly, the eigenvalues of $PP^*$ are also non-negative.

  Now if $P^*P$ has a positive eigenvalue $\lambda>0$, then $P^*Pu_1=\lambda u_1$ with non-zero vector $u_1\in H_1$. It follows that $PP^*(Pu_1)=\lambda (Pu_1)$. It is sufficient to prove that $Pu_1$ is non-zero. Indeed, if $Pu_1=0$, then $P^*Pu_1=0$ and $\lambda=0$ which contradicts with our assumption. Therefore, any positive eigenvalue of $P^*P$ is an eigenvalue of $PP^*$. Similarly, any positive eigenvalue of $PP^*$ is an eigenvalue of $P^*P$.

  Next, suppose that $PP^*$ is surjective.
  We show that $PP^*u_2=0$ has only the trivial solution $u_2=0$. In fact, $PP^*u_2=0$ implies that $\norm{P^*u_2}^2_{H_1}=\langle u_2, PP^*u_2\rangle_{H_2}=0$, i.e., $P^*u_2=0$. Thanks to the surjectivity of $PP^*$, there exists a vector $u_3\in H_2$ such that $u_2=PP^*u_3$. Let $u_1=P^*u_3\in H_1$. Hence $u_2=Pu_1$ and $P^*Pu_1=0$. Taking inner product with $u_1$, we have $\norm{Pu_1}^2_{H_2}=\langle u_1, P^*Pu_1\rangle_{H_1}=0$, i.e., $u_2=Pu_1=0$. Therefore $PP^*$ is injective. This with the surjectivity assumption of $PP^*$ leads to that $PP^*$ is invertible.
\end{proof}

\begin{rem}
  For the finite dimensional case $H_2=\R^M$, conditions for the operator $P$ in Lemma \ref{lem..spectrum.positive} are reduced to that the matrix of $P$ has rank $M$ (full rank).
\end{rem}

Given $g\in H_2$, we consider the following two problems.

(i)
The initial value problem
\begin{equation*}
  \left\{
    \begin{array}{ll}
      \dfrac{\D u}{\D t}=P^*(g-Pu)\\
      u(0)=u_{\rm ini}.
    \end{array}
  \right.
\end{equation*}
Since this equation is linear and with nonpositive eigenvalues on the right hand side, there exists a unique global-in-time solution $u(t)$ for all $t\in[0,+\infty)$ satisfying the initial condition. Moreover, the long-time limit $\lim_{t\rightarrow+\infty}u(t)$ exists and will be denoted as $u_\infty$.

(ii)
The minimization problem
\begin{align*}
  &\min_{u -u_{\rm ini}\in H_1}\norm{u-u_{\rm ini}}_{H_1},\\
  &\text{s.t.}\quad Pu=g.
\end{align*}
In the following, we will show it has a unique minimizer which is denoted as $u_{\min}$.

Now we show the following equivalent theorem.
\begin{thm}\label{thm..EquivalenceDynamicsMinimization}
  Suppose that $PP^*$ is surjective. The above Problems (i) and (ii) are equivalent in the sense that $u_\infty=u_{\min}$. 
  More precisely, we have
  \begin{equation}
    u_\infty=u_{\min}=P^*(PP^*)^{-1}(g-Pu_{\rm ini})+u_{\rm ini}.
  \end{equation}
\end{thm}
\begin{proof}
  Let $\tilde{u}=u-u_{\rm ini}$ and $\tilde{g}=g-Pu_{\rm ini}$. Then it is sufficient to show the following problems (i') and (ii') are equivalent.

  (i')
  The initial value problem
  \begin{equation*}
  \left\{
    \begin{array}{l}
      \dfrac{\D \tilde{u}}{\D t}
      = P^*(\tilde{g}-P\tilde{u})\\
      \tilde{u}(0)=0.
    \end{array}
  \right.
  \end{equation*}

  (ii')
  The minimization problem
  \begin{align*}
    &\min_{\tilde{u}}\norm{\tilde{u}}^2_{H_1},\\
    &\text{s.t.}\quad P\tilde{u}=\tilde{g}.
  \end{align*}

  We claim that $\tilde{u}_{\min}=P^*(PP^*)^{-1}\tilde{g}$. Thanks to Lemma \ref{lem..spectrum.positive}, $PP^*$ is invertible, and thus $u_{\min}$ is well-defined and satisfies that $P\tilde{u}=\tilde{g}$. It remains to show that this solution is unique. In fact, for any $\tilde{u}$ satisfying $P\tilde{u}=\tilde{g}$, we have
  \begin{align*}
    \langle\tilde{u}-\tilde{u}_{\min},\tilde{u}_{\min}\rangle_{H_1}
    &= \langle\tilde{u}-\tilde{u}_{\min},P^*(PP^*)^{-1}\tilde{g}\rangle_{H_1}\\
    &= \langle P(\tilde{u}-\tilde{u}_{\min}), (PP^*)^{-1}\tilde{g}\rangle_{H_2}\\
    &= \langle P\tilde{u}, (PP^*)^{-1}\tilde{g}\rangle_{H_2}-\langle P\tilde{u}_{\min}, (PP^*)^{-1}\tilde{g}\rangle_{H_2}\\
    &= 0.
  \end{align*}
  Therefore,
  \begin{equation*}
    \norm{\tilde{u}}^2_{H_1}=\norm{\tilde{u}_{\min}}^2_{H_1}+\norm{\tilde{u}-\tilde{u}_{\min}}^2_{H_1}\geq \norm{\tilde{u}_{\min}}^2_{H_1}.
  \end{equation*}
  The equality holds if and only if $\tilde{u}=\tilde{u}_{\min}$.

  For problem (i'), from the theory of ordinary differential equations on Hilbert spaces, we have that its solution can be written as
  \begin{equation*}
    \tilde{u}(t)=P^*(PP^*)^{-1}\tilde{g}+\sum_{i\in I}c_i v_i\exp(-\lambda_i t),
  \end{equation*}
  where $\lambda_i$, $i\in I$ are positive eigenvalues of $PP^*$, $I$ is an index set with at most countable cardinality, and $v_i$, $i\in I$ are eigenvectors in $H_1$.
  Thus $\tilde{u}_\infty=\tilde{u}_{\min}=P^*(PP^*)^{-1}\tilde{g}$.

  Finally, by back substitution, we have
  \begin{equation*}
    u_\infty=u_{\min}
    = P^*(PP^*)^{-1}\tilde{g}+u_0
    = P^*(PP^*)^{-1}(g-Pu_{\rm ini})+u_{\rm ini}.
  \end{equation*}
\end{proof}

The following corollaries are obtained directly from Theorem \ref{thm..EquivalenceDynamicsMinimization}.

\begin{cor}\label{cor..EquivalencdTheta}
  Let $u$ be the parameter vector $\theta$ in $H_1=\R^{N_p}$, $g$ be the outputs of the training data $Y$, and $P$ be a full rank matrix in the linear DNN model. Then 
  the following two problems are equivalent in the sense that $\theta_\infty=\theta_{\min}$.

  (A1)
  The initial value problem
  \begin{equation*}
    \left\{
      \begin{array}{l}
        \dfrac{\D \theta}{\D t}=P^*(Y-P\theta)\\
        \theta(0)=\theta_{\rm ini}.
      \end{array}
    \right.
  \end{equation*}

  (A2)
  The minimization problem
  \begin{align*}
    &\min_{\theta-\theta_{\rm ini}\in \R^{N_p}}\norm{\theta-\theta_{\rm ini}}_{\R^{N_p}},\\
    &\text{s.t.}\quad P\theta=Y.
  \end{align*}
\end{cor}

The next corollary is a weighted version of Theorem \ref{thm..EquivalenceDynamicsMinimization}.
\begin{cor}\label{cor..EquivalencdHW}
  Let $H_1$ and $H_2$ be two seperable Hilbert spaces and $\Gamma: H_1\rightarrow H_1$ be an injective operator.
  Define the Hilbert space $H_\Gamma:=\mathrm{Im}(\Gamma)$. 
  Let $g\in H_2$ and $P: H_\Gamma\rightarrow H_2$ be an operator such that $PP^*: H_2\to H_2$ is surjective.
  Then $\Gamma^{-1}: H_\Gamma\rightarrow H_1$ exists and $H_\Gamma$ is a Hilbert space with norm $\norm{u}_{H_\Gamma}:=\norm{\Gamma^{-1}u}_{H_1}$. Moreover, the following two problems are equivalent in the sense that $u_\infty=u_{\min}$.

  (B1)
  The initial value problem
  \begin{equation*}
    \left\{
      \begin{array}{l}
        \dfrac{\D u}{\D t}=\gamma^2P^*(g-Pu)\\
        u(0)=u_{\rm ini}.
      \end{array}
    \right.
  \end{equation*}

  (B2)
  The minimization problem
  \begin{align*}
    &\min_{u-u_0\in H_\Gamma}\norm{u-u_{\rm ini}}_{H_\Gamma},\\
    &\text{s.t.}\quad Pu=g.
  \end{align*}
\end{cor}
\begin{proof}
  The operator $\Gamma:H_1\rightarrow H_W$ is bijective. Hence $\Gamma^{-1}:H_\Gamma\rightarrow H_1$ is well-defined and $H_\Gamma$ with norm $\norm{\cdot}_{H_\Gamma}$ is a Hilbert space.
  The equivalence result holds by applying Theorem \ref{thm..EquivalenceDynamicsMinimization} with proper replacements. More precisely, we replace $u$ by $\Gamma u$ and $P$ by $P\Gamma$. 
\end{proof}

\begin{cor}\label{cor..EquivalencdHWFrequencyApp}
  Let $\gamma: \R^{d}\rightarrow\R^+$ be a positive function and $h$ be a function in $L^2(\R^{d})$. The operator $\Gamma: L^2(\R^{d})\rightarrow L^2(\R^{d})$ is defined by $[\Gamma\hat{h}](\xi)=\gamma(\xi)\hat{h}(\xi)$, $\xi\in\R^{d}$.
  Define the Hilbert space $H_\Gamma:=\mathrm{Im}(\Gamma)$. 
  Let $X=(x_i)_{i=1}^M\in \R^{d\times M}$, $Y=(y_i)_{i=1}^M \in \R^{M}$ and $P: H_\Gamma\rightarrow \R^{M}$ be a surjective operator
  \begin{equation}
  P: \hat{h}\mapsto \left(\int_{\R^{d}}\hat{h}(\xi)\E^{2\pi\I x_i\cdot \xi}\diff{\xi}\right)_{i=1}^M=(h(x_i))_{i=1}^M.
  \end{equation}
  Then the following two problems are equivalent in the sense that $\hat{h}_\infty=\hat{h}_{\min}$.

  (C1)
  The initial value problem
  \begin{equation*}
    \left\{
      \begin{array}{l}
        \dfrac{\D \hat{h}(\xi)}{\D t}=(\gamma(\xi))^2\sum_{i=1}^M(y_i\E^{-2\pi\I x_i\cdot\xi}-\hat{h}(\xi)*\E^{-2\pi\I x_i\cdot\xi})\\
        \hat{h}(0)=\hat{h}_{\rm ini}.
      \end{array}
    \right.
  \end{equation*}

  (C2)
  The minimization problem
  \begin{align*}
    &\min_{\hat{h}-\hat{h}_{\rm ini}\in H_\Gamma}\int_{\R^{d}}(\gamma(\xi))^{-2}|\hat{h}(\xi)-\hat{h}_{\rm ini}(\xi)|^2\diff{\xi},\\
    &\text{s.t.}\quad h(x_i)=y_i,\quad i=1,\cdots,M.
  \end{align*}
\end{cor}
\begin{proof}
  Let $H_1=L^2(\R^{d})$, $H_2=\R^{M}$, $u=\hat{h}$, and $g=Y$. By definition, $\Gamma$ is surjective. Then by Corollary \ref{cor..EquivalencdHW}, we have that $\Gamma^{-1}: H_\Gamma\rightarrow L^2(\R^{d})$ exists and $H_W$ is a Hilbert space with norm $\norm{\hat{h}}_{H_\Gamma}:=\norm{\Gamma^{-1}\hat{h}}_{L^2(\R^{d})}$. Moreover, $\norm{\hat{h}-\hat{h}_{\rm ini}}_{H_\Gamma}^2=\int_{\R^{d}}(\gamma(\xi))^{-2}\abs{\hat{h}(\xi)-\hat{h}_{\rm ini}(\xi)}^2\diff{\xi}$. We note that $[P^*Y](\xi)=\sum_{i=1}^{M} y_i\E^{-2\pi\I x_i\cdot \xi}$ for all $\xi\in\R^{d}$. Thus
  \begin{align*}
    [P^*P\hat{h}](\xi)
    &= \left[P^*\left(\int_{\R^{d}}\hat{h}(\xi')\E^{2\pi\I x_i\cdot \xi'}\diff{\xi'}\right)_{i=1}^M\right](\xi)\\
    &= \sum_{i=1}^M \int_{\R^{d}}\hat{h}(\xi')\E^{2\pi\I x_i\cdot \xi'}\diff{\xi'}\E^{-2\pi\I x_i\cdot \xi}\\
    &= \sum_{i=1}^M \int_{\R^{d}}\hat{h}(\xi')\E^{-2\pi\I x_i\cdot (\xi-\xi')}\diff{\xi'}\\
    &= \sum_{i=1}^M \hat{h}(\xi)*\E^{-2\pi\I x_i\cdot \xi}.
  \end{align*}
  The equivalence result then follows from Corollary \ref{cor..EquivalencdHW}.
\end{proof}

We remark that $P^*P\hat{h}=\sum_{i=1}^M\widehat{h\delta_{x_i}}$, where $\delta_{x_i}(\cdot)=\delta(\cdot-x_i)$, $i=1,\cdots,M$. Therefore problem (C1) can also be written as:
  \begin{equation*}
    \left\{
      \begin{array}{l}
        \dfrac{\D \hat{h}}{\D t}=\gamma^2\sum_{i=1}^M(y_i\widehat{\delta_{x_i}}-\widehat{h\delta_{x_i}}
        )\\
        \hat{h}(0)=\hat{h}_{\rm ini}.
      \end{array}
    \right.
  \end{equation*}

In the following, we study the discretized version of this dynamics-optimization problem (C1\&C2).

\begin{cor}\label{cor..EquivalencdHWFrequencyDiscrete}
  Let $\gamma: \Z^{d}\rightarrow\R^+$ be a positive function defined on lattice $\Z^{d}$. The operator $\Gamma: \ell^2(\Z^{d})\rightarrow \ell^2(\Z^{d})$ is defined by $[\Gamma\hat{h}](k)=\gamma(k)\hat{h}(k)$, $k\in\Z^{d}$. Here $\ell^2(\Z^{d})$ is set of square summable functions on the lattice $\Z^{d}$.
  Define the Hilbert space $H_\Gamma:=\mathrm{Im}(\Gamma)$. 
  Let $X=(x_i)_{i=1}^M\in \mathbb{T}^{d\times M}$, $Y=(y_i)_{i=1}^M \in \R^{M}$ and $P: H_\Gamma\rightarrow \R^{M}$ be a surjective operator such as 
  \begin{equation}
  P: \hat{h}\mapsto \left(\sum_{k\in\Z^{d}}\hat{h}(k)\E^{2\pi\I x_i\cdot k}\right)_{i=1}^M.
  \end{equation}
  Then the following two problems are equivalent in the sense that $\hat{h}_\infty=\hat{h}_{\min}$.

  (D1)
  The initial value problem
  \begin{equation*}
    \left\{
      \begin{array}{ll}
        \dfrac{\D \hat{h}(k)}{\D t}=(\gamma(k))^2\sum_{i=1}^M(y_i\E^{-2\pi\I x_i\cdot k}-\hat{h}(k)*\E^{-2\pi\I x_i\cdot k})\\
        \hat{h}(0)=\hat{h}_{\rm ini}.
      \end{array}
    \right.
  \end{equation*}

  (D2)
  The minimization problem
  \begin{align*}
    &\min_{\hat{h}-\hat{h}_{\rm ini}\in H_\Gamma}\sum_{k\in\Z^{d}}(\gamma(k))^{-2}\abs{\hat{h}(k)-\hat{h}_{\rm ini}(k)}^2,\\
    &\text{s.t.}\quad h(x_i)=y_i,\quad i=1,\cdots,M.
  \end{align*}
\end{cor}
\begin{proof}
  Let $H_1=\ell^2(\Z^{d})$, $H_2=\R^{M}$, $u=\hat{h}$, and $g=Y$. By definition, $\Gamma$ is surjective. Then by Corollary \ref{cor..EquivalencdHW}, we have that $\Gamma^{-1}: H_\Gamma\rightarrow \ell^2(\Z^{d})$ exists and $H_\Gamma$ is a Hilbert space with norm $\norm{\hat{h}}_{H_\Gamma}:=\norm{\Gamma^{-1}\hat{h}}_{\ell^2(\Z^{d})}$. Moreover, $\norm{\hat{h}-\hat{h}_{\rm ini}}_{H_\Gamma}^2=\sum_{k\in\Z^{d}}(w(k))^{-2}\abs{\hat{h}(k)-\hat{h}_{\rm ini}(k)}^2$. 
  We note that $[P^*Y](k)=\sum_{i=1}^{M} y_i\E^{-2\pi\I x_i\cdot k}$ for all $k\in\Z^{d}$. Thus
  \begin{align*}
    [P^*P\hat{h}](k)
    &= \left[P^*\left(\sum_{k'\in\Z^{d}}\hat{h}(k')\E^{2\pi\I x_i\cdot k'}\right)_{i=1}^M\right](k)\\
    &= \sum_{i=1}^M \sum_{k'\in\Z^{d}}\hat{h}(k')\E^{2\pi\I x_i\cdot k'}\E^{-2\pi\I x_i\cdot k}\\
    &= \sum_{i=1}^M \sum_{k'\in\Z^{d}}\hat{h}(k')\E^{-2\pi\I x_i\cdot (k-k')}\\
    &= \sum_{i=1}^M \hat{h}(k)*\E^{-2\pi\I x_i\cdot k}.
  \end{align*}
  The equivalence result then follows from Corollary \ref{cor..EquivalencdHW}.
\end{proof}

\section{Proof of the \textit{a priori} generalization error bound \label{proofapriori}}
\subsection{Problem Setup}

We focus on regression problem. Assume the target function $f:\Omega:=[0,1]^{d}\to\R$.
Let the training set be $\{(x_{i}; y_{i})\}_{i=1}^{M}$,
where $x_{i}$'s are independently sampled from an underlying
distribution $\rho(x)$ and $y_{i}=f(x_{i})$. We consider
the square loss 
\begin{equation}
  \ell(x;h)=\abs{h(x)-f(x)}^{2},
\end{equation}
with population risk 
\begin{equation}
  L=L(f)=\mathbb{E}_{x\sim\rho}\ell(x;h)
\end{equation}
and empirical risk 
\begin{equation}
  \tilde{L}_{M}(h)=\frac{1}{M}\sum_{i=1}^{M}\ell(x_{i};h).
\end{equation}

\subsection{FP-space}
We denote $\Z^{d*}:=\Z^d\backslash\{0\}$.
Given a frequency weight function $\gamma: \Z^{d}\to \R^+$ or $\gamma: \Z^{d*}\to \R^+$ satisfying
\begin{equation}
  \norm{\gamma}_{\ell^2}=\left(\sum_{k\in\Z^{d}}(\gamma(k))^2\right)^{\frac{1}{2}}<+\infty\quad\text{or}\quad \norm{\gamma}_{\ell^2}=\left(\sum_{k\in\Z^{d*}}(\gamma(k))^2\right)^{\frac{1}{2}}<+\infty,
\end{equation}
we define the FP-norm for all function $h\in L^2(\Omega)$: 
\begin{equation}
  \norm{h}_{\gamma}:=\norm{\hat{h}}_{H_\Gamma}=\left(\sum_{k\in\Z^{d}}(\gamma(k))^{-2}|\hat{h}(k)|^{2}\right)^{\frac{1}{2}}.
\end{equation}
If $\gamma:\Z^{d*}\to\R^+$ is not defined at $0$, we set $(\gamma(0))^{-1}:=0$ in the above definition and $\norm{\cdot}_\gamma$ is only a semi-norm of $h$.

Then we define the FP-space
\begin{equation}
  F_{\gamma}(\Omega)=\{h\in L^2(\Omega):\norm{h}_{\gamma}<\infty\}.
\end{equation}
By our definition, the FP-space is always a subspace of $L^2(\Omega)$. In addition, if $\gamma: k\mapsto |k|^{-m/2}$ for $k\in\Z^{d*}$, then functions in the FP-space with $\hat{h}(0)=\int_{\Omega} h(x) \diff{x}=0$ form the Sobolev space $H^m(\Omega)$.

\subsection{A prior generalization error bound }
\begin{lem}\label{Rad bound}
  (i) For $\mathcal{H}=\{h:\norm{h}_{\gamma}\leq Q\}$ with $\gamma: \Z^{d}\to \R^+$, we have
  \begin{equation}
    \tilde{R}(\mathcal{H})\leq\frac{1}{\sqrt{M}}Q\norm{\gamma}_{\ell^2}.
  \end{equation}
  (ii) For $\mathcal{H}'=\{h:\norm{h}_{\gamma}\leq Q, \abs{\hat{h}(0)}\leq c_{0}\}$ with $\gamma: \Z^{d*}\to \R^+$ and $\gamma^{-1}(0):=0$, we have
  \begin{equation}
    \tilde{R}(\mathcal{H}')\leq \frac{c_{0}}{\sqrt{M}}+\frac{1}{\sqrt{M}}Q\norm{\gamma}_{\ell^2}.
  \end{equation}

\end{lem}

\begin{proof}
  We first prove (ii) since it is more involved.
  By the definition of the Rademacher complexity 
  \begin{equation}
    \tilde{R}(\mathcal{H}')=\frac{1}{M}\mathbb{E}_{\e}\left[\sup_{h\in\mathcal{H}'}\sum_{i=1}^{M}\e_{i}h(x_{i})\right].
  \end{equation}
  Let $\e(x)=\sum_{i=1}^{M}\e_{i}\delta(x-x_{i})$,
  where $\e_{i}$'s are i.i.d. random variables with $\mathbb{P}(\e_i=1)=\mathbb{P}(\e_i=-1)=\frac{1}{2}$. We have $\hat{\e}(k)=\int_{\Omega}\sum_{i=1}^M\e_i\delta(x-x_i)\E^{-2\pi\I k\cdot x}\diff{x}=\sum_{i=1}^M\e_i\E^{-2\pi\I k\cdot x_i}$. Note that
  \begin{align}
    \sup_{h\in\mathcal{H}'}\sum_{i=1}^{M}\e_{i}h(x_{i})
    = \sup_{h\in\mathcal{H}'}\sum_{i=1}^{M}\e_{i}\bar{h}(x_{i})
    &= \sup_{h\in\mathcal{H}'}\sum_{i=1}^{M}\e_{i}\sum_{k\in\Z^d}\overline{\hat{h}(k)}\E^{-2\pi\I k\cdot x_i}\\
    &= \sup_{h\in\mathcal{H}'}\sum_{k\in\Z^d}\hat{\e}(k)\overline{\hat{h}(k)}.
  \end{align}
  By the Cauchy--Schwarz inequality,
  \begin{align}
    \sup_{h\in\mathcal{H}'}\sum_{k\in\Z^{d}}\hat{\e}(k)\overline{\hat{h}(k)} 
    & \leq \sup_{h\in\mathcal{H}}\left[\hat{\e}(0)\overline{\hat{h}(0)}+\left(\sum_{k\in\Z^{d*}}(w(k))^2\abs{\hat{\e}(k)}^{2}\right)^{1/2}\left(\sum_{k\in\Z^{d*}}(w(k))^{-2}\abs{\overline{\hat{h}(k)}}^{2}\right)^{1/2}\right]\\
    & \leq c_{0}\abs{\hat{\e}(0)}+Q\left(\sum_{k\in\Z^{d*}}(w(k))^2\abs{\hat{\e}(k)}^{2}\right)^{1/2}.
  \end{align}
  Since $\mathbb{E}_{\e}\abs{\hat{\e}(0)}\leq (\mathbb{E}_{\e}\abs{\hat{\e}(0)}^2)^{1/2}=\sqrt{M}$, $\mathbb{E}_{\e}\abs{\hat{\e}(k)}^{2}=\mathbb{E}_{\e}\sum_{i,j=1}^{M}\e_{i}\e_{j}\E^{-2\pi\mathrm{i}k\cdot(x_{i}-x_{j})}=M$, we obtain
  \begin{align}
    \mathbb{E}_{\e}\left[\sup_{h\in\mathcal{H}'}\sum_{i=1}^{M}\e_{i}h(x_{i})\right] 
    &\leq c_{0}\sqrt{M}+Q \mathbb{E}_{\e}\left(\sum_{k\in\Z^{d*}}(w(k))^2\abs{\hat{\e}(k)}^{2}\right)^{1/2}\\
    &\leq c_{0}\sqrt{M}+Q \left(\mathbb{E}_{\e}\sum_{k\in\Z^{d*}}(w(k))^2\abs{\hat{\e}(k)}^{2}\right)^{1/2}\\
    &= c_{0}\sqrt{M}+Q\sqrt{M}\norm{w}_{\ell^2}.
  \end{align}
  This leads to 
  \begin{equation}
    \tilde{R}(\mathcal{H}')\leq\frac{c_{0}}{\sqrt{M}}+\frac{1}{\sqrt{M}}Q\norm{w}_{\ell^2}.
  \end{equation}

  For (ii), the proof is similar to (i). We have
  \begin{equation}
    \mathbb{E}_{\e}\left[\sup_{h\in\mathcal{H}}\sum_{k\in\Z^{d}}\hat{\e}(k)\overline{\hat{h}(k)}\right]
    \leq Q\mathbb{E}_{\e}\left(\sum_{k\in\Z^{d}}(w(k))^2|\hat{\e}(k)|^{2}\right)^{1/2}
    \leq Q\sqrt{M}\norm{w}_{\ell^2}.
  \end{equation}
  Therefore
  \begin{equation}
    \tilde{R}(\mathcal{H})\leq \frac{1}{\sqrt{M}}Q\norm{w}_{\ell^2}.
  \end{equation}
\end{proof}

\begin{lem}\label{FPnorm Bound}
    Suppose that the real-valued target function $f\in F_\gamma(\Omega)$ and that the training dataset $\{x_{i}; y_i\}_{i=1}^{M}$ satisfies $y_i=f(x_{i})$, $i=1,\cdots,M$. If $\gamma: \Z^{d}\to \R^+$, then there exists a unique solution $h_{M}$ to the regularized model
  \begin{equation}
    \min_{h-h_{\rm ini}\in F_
    \gamma(\Omega)} \norm{h-h_{{\rm ini}}}_{\gamma},\quad\text{s.t.}\quad h(x_i)=y_i,\quad i=1,\cdots,M.\label{eq..FPnormBoundMinimizationProblem}
  \end{equation}
  Moreover, we have
  \begin{equation}
    \norm{h_{M}-h_{{\rm ini}}}_\gamma\leq \norm{f-h_{{\rm ini}}}_\gamma.
  \end{equation}
\end{lem}

\begin{proof}
    By the definition of the FP-norm, we have $\norm{h_M-h_{\rm ini}}_\gamma=\norm{\hat{h}_M-\hat{h}_{\rm ini}}_{H_\Gamma}$. According to Corollary \ref{cor..EquivalencdHWFrequencyDiscrete}, the minimizer of problem \eqref{eq..FPnormBoundMinimizationProblem} exists, i.e., $h_M$ exists. Since the target function $f(x)$ satisfies the constraints $f(x_i)=y_i$, $i=1,\cdots,M$, we have $\norm{h_{M}-h_{{\rm ini}}}_\gamma\leq \norm{f-h_{{\rm ini}}}_\gamma$.
\end{proof}

\begin{lem}\label{0freqconstraint}
    Suppose that the real-valued target function $f\in F_\gamma(\Omega)$ and the training dataset $\{x_{i}; y_i\}_{i=1}^{M}$ satisfies $y_i=f(x_{i})$, $i=1,\cdots,M$. If $\gamma: \Z^{d*}\to \R^+$ with $\gamma^{-1}(0):=0$, then there exists a solution $h_{M}$ to the regularized model
  \begin{equation}
    \min_{h-h_{\rm ini}\in F_
    \gamma(\Omega)} \norm{h-h_{{\rm ini}}}_{\gamma},\quad\text{s.t.}\quad h(x_i)=y_i,\quad i=1,\cdots,M.
  \end{equation}
  Moreover, we have 
  \begin{equation}
    \abs{\widehat{\left(h_{M}-h_{{\rm ini}}\right)}(0)}
    \leq \norm{f-h_{{\rm ini}}}_{\infty}+\norm{f-h_{{\rm ini}}}_{\gamma}\norm{\gamma}_{\ell^2}.
  \end{equation}
\end{lem}

\begin{proof}
  Let $f'=f-h_{{\rm ini}}$.
  Since $h_{M}(x_{i})-f(x_{i})=0$
  for $i=1,\cdots,M$, we have $h_{M}(x_{i})-f'(x_{i})-h_{{\rm ini}}(x_{i})=0$.
  Therefore
  \begin{align}
    \abs{\widehat{\left(h_{M}-h_{{\rm ini}}\right)}(0)}
    &= \abs{f'(x_{i})-\sum_{k\in\Z^{d*}}\widehat{\left(h_{M}-h_{{\rm ini}}\right)}(k)\E^{2\pi\mathrm{i}k\cdot x_{i}}}\\
    &\leq \norm{f'}_{\infty}+\sum_{k\in\Z^{d*}}\abs{\widehat{\left(h_{M}-h_{{\rm ini}}\right)}(k)}\\
    &\leq \norm{f'}_{\infty}+\left(\sum_{k\in\Z^{d*}}(\gamma(k))^2\right)^{\frac{1}{2}}\left(\sum_{k\in\Z^{d*}}(\gamma(k))^{-2}\abs{\widehat{\left(h_{M}-h_{{\rm ini}}\right)}(k)}^{2}\right)^{\frac{1}{2}}\\
    &\leq \norm{f'}_{\infty}+\norm{h_M-h_{\rm ini}}_{\gamma}\norm{\gamma}_{\ell^2}\\
    &\leq \norm{f'}_{\infty}+\norm{f'}_{\gamma}\norm{\gamma}_{\ell^2}.
  \end{align}
  We remark that the last step is due to the same reason as Lemma \ref{FPnorm Bound}.
\end{proof}

\begin{thm}
  Suppose that the real-valued target function $f\in F_\gamma(\Omega)$, the training dataset $\{x_{i}; y_i\}_{i=1}^{M}$ satisfies $y_i=f(x_{i})$, $i=1,\cdots,M$, and $h_{M}$ is the solution of the regularized model
  \begin{equation}
    \min_{h-h_{\rm ini}\in F_
    \gamma(\Omega)} \norm{h-h_{{\rm ini}}}_{\gamma},\quad\text{s.t.}\quad h(x_i)=y_i,\quad i=1,\cdots,M.\label{eq:optf}
  \end{equation}
  Then we have

  (i) given $\gamma: \Z^{d}\to \R^+$, 
  for any $\delta\in(0,1)$, with probability at least $1-\delta$ over
  the random training sample, the population risk has the bound
  \begin{equation}
    L(h_{M})
    \leq \norm{f-h_{{\rm ini}}}_{\gamma}\norm{\gamma}_{\ell^2}
    \left(\frac{2}{\sqrt{M}}+4\sqrt{\frac{2\log(4/\delta)}{M}}\right).
  \end{equation}

  (ii) given $\gamma: \Z^{d*}\to \R^+$ with $\gamma(0)^{-1}:=0$, for any $\delta\in(0,1)$,
  with probability at least $1-\delta$ over the random training sample,
  the population risk has the bound
  \begin{equation}
    L(h_{M})
    \leq \left(\norm{f-h_{\rm ini}}_{\infty}+2\norm{f-h_{{\rm ini}}}_{\gamma}\norm{\gamma}_{\ell^2}
    \right)
    \left(\frac{2}{\sqrt{M}}+4\sqrt{\frac{2\log(4/\delta)}{M}}\right).
  \end{equation}
\end{thm}

\begin{proof}
  Let $f'=f-h_{{\rm ini}}$.

  (i) Given $\gamma:\Z^d\to\R^+$, we set $\mathcal{H}=\{h: \norm{h-h_{{\rm ini}}}_{\gamma}\leq \norm{f'}_{\gamma}\}$.
  According to Lemma \ref{FPnorm Bound}, the solution of problem (\ref{eq:optf}) $h_{M}\in{\cal H}$. By the
  relation between generalization gap and Rademacher complexity (\citet{bartlett2002rademacher,shalev2014understanding}),
  \begin{equation}
  \abs{L(h_{M})-\tilde{L}_M(h_{M})}\leq 2\tilde{R}(\mathcal{H})+2\sup_{h,h'\in\mathcal{H}}\norm{h-h'}_{\infty}\sqrt{\frac{2\log(4/\delta)}{M}}.
  \end{equation}
  One of the component can be bounded as follows\textcolor{blue}{{} }
  \begin{align}
    \sup_{h,h'\in\mathcal{H}}\norm{h-h'}_{\infty} 
    & \leq \sup_{h\in\mathcal{H}}2\norm{h-h_{{\rm ini}}}_{\infty}\\
    & \leq \sup_{h\in\mathcal{H}}2\max_{x}(\abs{\sum_{k\in\Z^{d}}\widehat{\left(h-h_{{\rm ini}}\right)}(k)\E^{2\pi\I k\cdot x}})\\
    & \leq \sup_{h\in\mathcal{H}}2\sum_{k\in\Z^{d}}\abs{\widehat{\left(h-h_{{\rm ini}}\right)}(k)}\\
    & \leq 2\sup_{h\in\mathcal{H}}\left(\sum_{k\in\Z^{d}}(\gamma(k))^2\right)^{\frac{1}{2}}\left(\sum_{k\in\Z^{d}}(\gamma(k))^{-2}\abs{\widehat{\left(h-h_{{\rm ini}}\right)}(k)}^{2}\right)^{\frac{1}{2}}\\
    & \leq 2\norm{f'}_{\gamma}\norm{\gamma}_{\ell^2}.
  \end{align}
  By Lemma \ref{Rad bound},
  \begin{equation}
    \tilde{R}(\mathcal{H})\leq \frac{1}{\sqrt{M}}\norm{f'}_{\gamma}\norm{\gamma}_{\ell^2}.
  \end{equation}
  By optimization problem (\ref{eq:optf}), $\tilde{L}(h_{M})\leq\tilde{L}(f')=0$.
  Therefore we obtain 
  \begin{equation}
    L(h)\leq \frac{2}{\sqrt{M}}\norm{f'}_{\gamma}\norm{\gamma}_{\ell^2}+4\norm{f'}_{\gamma}\norm{\gamma}_{\ell^2}\sqrt{\frac{2\log(4/\delta)}{M}}.
  \end{equation}

  (ii) Given $\gamma: \Z^{d*}\to \R^+$ with $\gamma(0)^{-1}:=0$, by Lemma \ref{Rad bound},
  \ref{FPnorm Bound}, and \ref{0freqconstraint}, define $\mathcal{H}'=\{h:\norm{h-h_{{\rm ini}}}_{\gamma}\leq \norm{f'}_{\gamma},\abs{\widehat{\left(h-h_{{\rm ini}}\right)}(0)}\leq \norm{f'}_{\infty}+\norm{f'}_{\gamma}\norm{\gamma}_{\ell^2}\}$,
  we obtain 
  \begin{equation}
    \tilde{R}(\mathcal{H}')
    \leq \frac{1}{\sqrt{M}}\norm{f'}_{\infty}+\frac{2}{\sqrt{M}}\norm{f'}_{\gamma}\norm{\gamma}_{\ell^2}.
  \end{equation}
  Also 
  \begin{align}
    \sup_{h,h'\in\mathcal{H}'}\norm{h-h'}_{\infty} & \leq \sup_{h\in\mathcal{H}}2\sum_{k\in\Z^{d}}\abs{\widehat{\left(h-h_{{\rm ini}}\right)}(k)}\\
    & \leq 2\sup_{h\in\mathcal{H}}\left[\abs{\widehat{\left(h-h_{{\rm ini}}\right)}(0)}+\left(\sum_{k\in\Z^{d*}}(\gamma(k))^2\right)^{\frac{1}{2}}\left(\sum_{k\in\Z^{d*}}(\gamma(k))^{-2}\abs{\widehat{\left(h-h_{{\rm ini}}\right)}(k)}^{2}\right)^{\frac{1}{2}}\right]\\
    & \leq 2\norm{f'}_{\infty}+4\norm{f'}_{\gamma}\norm{\gamma}_{\ell^2}.
  \end{align}
  Then 
  \begin{equation}
    L(h_{M})\leq \frac{2}{\sqrt{M}}\norm{f'}_{\infty}+\frac{4}{\sqrt{M}}\norm{f'}_{\gamma}\norm{\gamma}_{\ell^2}+\left(4\norm{f'}_{\infty}+8\norm{f'}_{\gamma}\norm{\gamma}_{\ell^2}\right)\sqrt{\frac{2\log(4/\delta)}{M}}.
  \end{equation}
\end{proof}
\begin{rem}
  By the assumption in the theorem, the target function $f$ belongs to $F_\gamma(\Omega)$ which is a subspace of $L^2(\Omega)$. In most applications, $f$ is also a continuous function. In any case, $f$ can be well-approximated by a large neural network due to universal approximation theory \citet{cybenko1989approximation}.
\end{rem}

\section{Numerical solution of the LFP model\label{sec:Numerical-solution-of}}

Numerically, we solve the following ridge regression problem (\citet{mei2019mean})
to approximate the solution of the optimization problem \eqref{eq: minFPnorm}
\begin{equation}
\min_{h}\sum_{i=1}^{M}\left(h(x_{i})-y_{i}\right)^{2}+\varepsilon\sum_{\xi\in \mathbb{L}^{d}}\left(\frac{\frac{1}{N}\sum_{i=1}^{N}\left(|r_{i}(0)|^{2}+w_{i}(0)^{2}\right)}{|\xi|^{d+3}}+\frac{4\pi^{2}\frac{1}{N}\sum_{i=1}^{N}\left(|r_{i}(0)|^{2}w_{i}(0)^{2}\right)}{|\xi|^{d+1}}\right)^{-1}|\hat{h}(\xi)|^{2},\label{eq:regopt}
\end{equation}
where the truncated lattice $\mathbb{L}^{d}=\frac{1}{L^{'}}\left(\Z^d\cap[-K+1,K-1]^{d}\right)$. $\varepsilon$
is fixed to $10^{-6}$. The numerical error contributed by $\epsilon$
is very small in a proper range. For the case of $d=1$ as shown in
Fig. \ref{fig:2relu}, we choose $L'=20$, $K=2000$ for the computation
of the solution of problem \eqref{eq:regopt}. For the case of $d=2$
as shown in Figs. \ref{fig:2relu-1} and \ref{fig:2relu-1-1}, we
choose $L'=24$, $K=120$ for the computation of solution. For higher
dimensional cases, because the size of set $\mathbb{L}^{d}$ grows exponentially
with $d$, we will encounter curse of dimensionality for solving
Eq. \eqref{eq:regopt}. Therefore, we do not test the LFP model for $d>2$ in this
work.
\begin{center}
\begin{figure}
\begin{centering}
\subfloat[]{\begin{centering}
\includegraphics[scale=0.2]{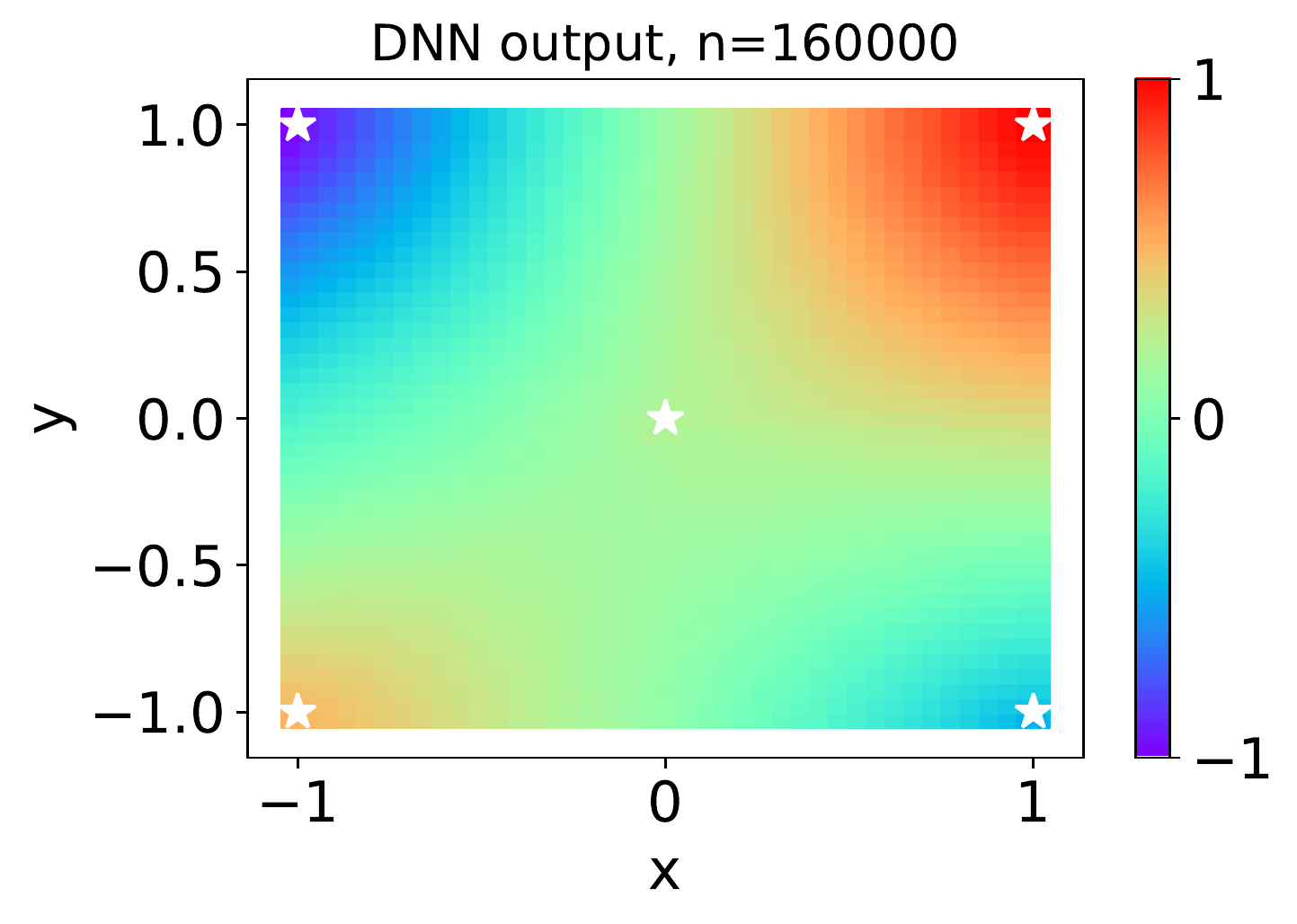} 
\par\end{centering}
}\subfloat[]{\begin{centering}
\includegraphics[scale=0.2]{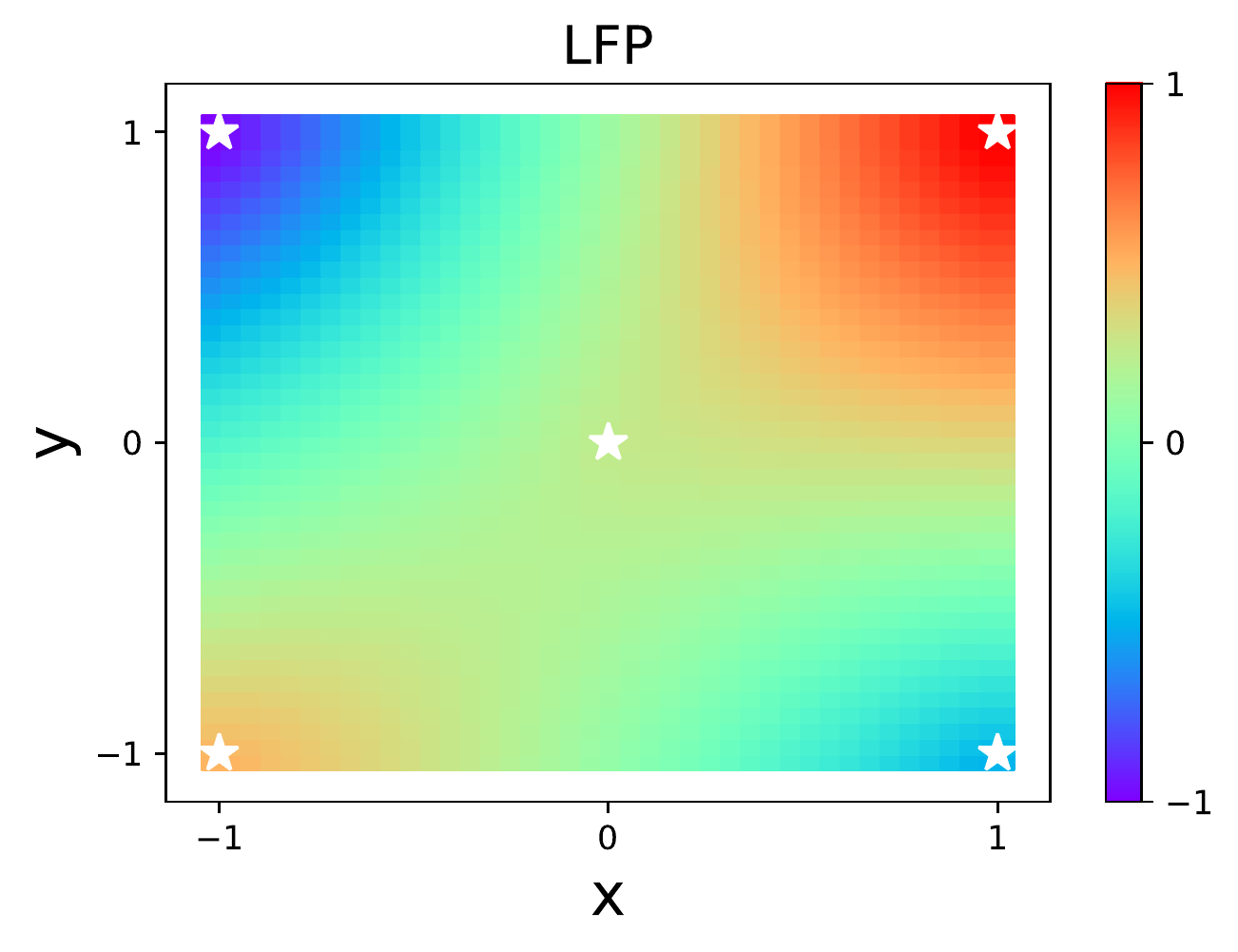} 
\par\end{centering}
}\subfloat[]{\begin{centering}
\includegraphics[scale=0.2]{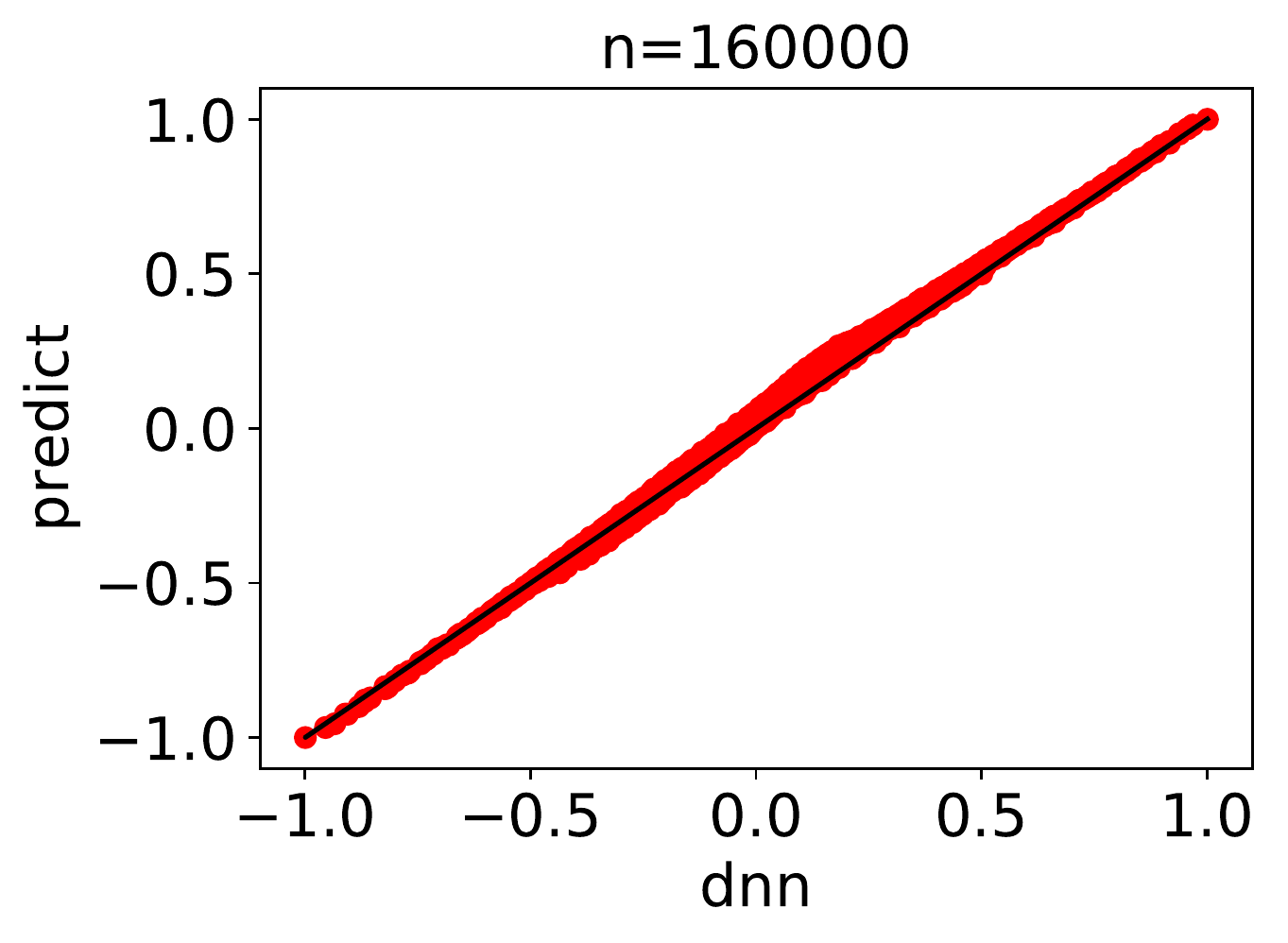} 
\par\end{centering}
}
\par\end{centering}
\caption{Same as Fig. \ref{fig:2relu-1} (a,b,c) except that the training data
consists of $5$ points and is asymmetrical. \label{fig:2relu-1-1} }
\end{figure}
\par\end{center}

\end{document}